\documentclass[twoside]{article}

\usepackage[accepted]{aistats2022}

\usepackage[round]{natbib}

\usepackage{xr}
\makeatletter
\newcommand*{\addFileDependency}[1]{%
  \typeout{(#1)}
  \@addtofilelist{#1}
  \IfFileExists{#1}{}{\typeout{No file #1.}}
}
\makeatother

\usepackage{hyperref}       %
\usepackage{url}            %
\usepackage{booktabs}       %
\usepackage{amsfonts}       %
\usepackage{nicefrac}       %
\usepackage{microtype}      %
\usepackage{xcolor}         %
\usepackage{MnSymbol}
\usepackage{tikz}
\usetikzlibrary{bayesnet}
\usetikzlibrary{arrows}
\usepackage{color}
\usepackage{graphicx}
\usepackage{caption}
\usetikzlibrary{backgrounds}
\usepackage{subcaption}
\usepackage{graphicx}
\usepackage{bbm}
\usepackage{enumitem}
\usepackage{float}

\newcommand{\B}[1]{\boldsymbol{#1}}

\usepackage[resetlabels]{multibib}
\newcites{Main,Appendix}{References,References}

\newcommand{\indep}{\perp \!\!\! \perp}

\usepackage[resetlabels]{multibib}
\newcites{Main,Appendix}{References,References (Appendices)}

\usepackage{amsmath}
\usepackage{amsthm}
\allowdisplaybreaks

\makeatletter
\newcommand{\newreptheorem}[2]{\newtheorem*{rep@#1}{\rep@title}\newenvironment{rep#1}[1]{\def\rep@title{#2 \ref*{##1}}\begin{rep@#1}}{\end{rep@#1}}}
\makeatother

\newcommand{\wass}{\text{Wass}}
\newcommand{\mmd}{\text{MMD}}

\newtheorem{definition}{Definition}
\newtheorem{theorem}{Proposition}
\newreptheorem{theorem}{Proposition}
\newtheorem{corollary}{Corollary}[theorem]
\newreptheorem{corollary}{Corollary}

\begin{document}

\runningauthor{Oscar Clivio, Fabian Falck, Brieuc Lehmann, George Deligiannidis, Chris Holmes}

\twocolumn[
\aistatstitle{Neural Score Matching for High-Dimensional Causal Inference}

\aistatsauthor{ Oscar Clivio$^{1}$ \hspace{1.8cm} Fabian Falck$^{1}$ }
\aistatsauthor{ Brieuc Lehmann$^{2}$ \hspace{1.2cm}
George Deligiannidis$^{1}$ \hspace{1.2cm} Chris Holmes$^{1,3}$ }

\aistatsaddress{ $^{1}$University of Oxford \hspace{0.4cm} $^{2}$University College London \hspace{0.4cm} $^{3}$Alan Turing Institute }
]

\begin{abstract}
Traditional methods for matching in causal inference are impractical for high-dimensional datasets. They suffer from the curse of dimensionality: exact matching and coarsened exact matching find exponentially fewer matches as the input dimension grows, and propensity score matching may match highly unrelated units together. 
To overcome this problem, we develop theoretical results which motivate the use of neural networks to obtain non-trivial, multivariate balancing scores of a chosen level of coarseness, in contrast to the classical, scalar propensity score. 
We leverage these balancing scores to perform matching for high-dimensional causal inference and call this procedure \emph{neural score matching}.
We show that our method is competitive against other matching approaches on semi-synthetic high-dimensional datasets, both in terms of treatment effect estimation and reducing imbalance.

\end{abstract}

\section{INTRODUCTION}
\label{sec:Introduction}

Estimating the causal effect of a treatment or a policy is the fundamental task of causal inference. 
For binary treatments, the quantity of interest is the difference between the outcome of a subject receiving a treatment (a \textit{treated} subject) and the outcome of that subject in the absence of treatment (a \textit{control} subject). 
The main difficulty in estimating a causal effect from observational data is that one cannot observe the outcome of both the true and the alternative scenario for the same subject -- also called the factual and counterfactual outcomes. 
For instance, to evaluate the effect of a lockdown on reducing infection case numbers in a given country, one cannot create an exact copy of that country to study the consequences of its absence. 

One possible solution would be to find a country that is very similar to the country under study, yet which did not experience a lockdown. 
This is the general idea behind \textit{matching} whereby each treated subject in the sample data is assigned to one or more subjects from the control group~\citepMain{stuart2010matching}. 
Matching is among the dominant techniques used in medicine and other domains to estimate the effect of a treatment from observational data \citepMain{su2019effect,farzadfar2012effectiveness,razonable2021casirivimab,webb2020outcomes}. 
Besides estimating the treatment effect, matching can serve additional objectives. 
For example, matching can reduce imbalance%
, i.e. distributional differences between the treated and control groups that indicate confounding and consequently make treatment effect estimation more difficult. 
Matching can also help to decrease costs by reducing the number of control samples required when the collection of data (e.g. subjects' outcome) is expensive~\citepMain{stuart2010matching}. 
Matching methods, however, generally suffer from the \textit{curse of dimensionality}~\citepMain{abadie2006large, Roberts2020textmatching}, rendering them impractical for many modern high-dimensional datasets, such as electronic health records or clinical images. 

\begin{figure*}[t]
    \centering
    \includegraphics[width=\linewidth]{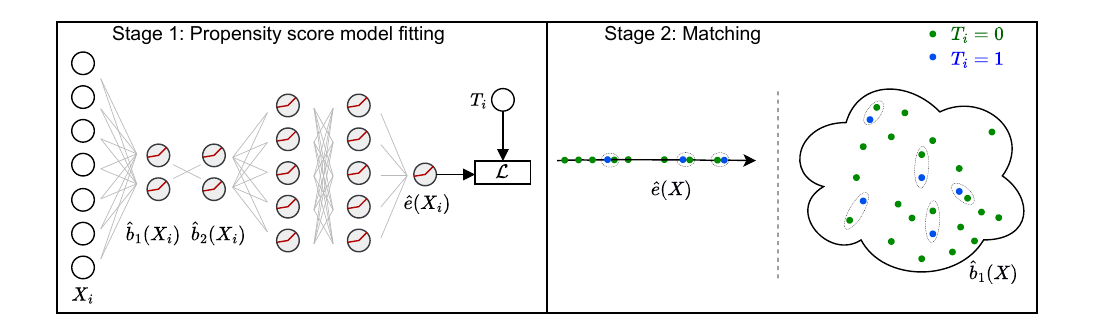}
    \caption{An illustration of neural score matching. 
    In the first stage, a propensity score model is fitted to obtain low-dimensional balancing scores. 
    In the second stage, samples are matched (to one neighbour) in the balancing score space based on a given distance metric. 
    Matched samples are subsequently used to estimate the ATT.}
    \label{fig:intuition}
\end{figure*}

In this work, we address the curse of dimensionality by first compressing the input covariates into a lower-dimensional matching space with a neural network and then matching in this space. 
Our contributions are as follows: 
(a) We develop novel theoretical results that bound the imbalance in the original covariate space via imbalance in a lower-dimensional balancing score space. 
We also extend these results to functions of covariates that violate the balancing score condition and which we refer to as ``non-balancing scores''. %
(b) These theoretical results motivate \textit{neural score matching}, a procedure to match on low-dimensional balancing scores obtained from the intermediate layers of a neural network modelling the propensity score. 
This yields a simple method for estimating average or group-based treatment effects in the presence of high-dimensional covariates without regressing on outcomes.
The intuition of neural score matching is illustrated in Fig.~\ref{fig:intuition}.
(c) We show that neural score matching is competitive against other matching methods on two causal inference benchmarks in terms of calibration error, treatment effect estimation and balance.

\section{MATCHING IN CAUSAL INFERENCE}
\label{sec:Matching in causal inference}

\subsection{Problem Setup}
\label{sec:Problem setup}

Let $(X_i,T_i,Y_i) \sim P$ be a dataset where $X_i$ denotes (pre-treatment) covariates, $T_i$ is the binary variable indicating whether the treatment under scrutiny has been applied to the subject or not, and $Y_i$ is the observed outcome after the treatment or absence of treatment, all corresponding to subject $i$. 
In the potential outcomes framework~\citepMain{Rubin2005po}, $Y_i(1)$ is the outcome which would have happened (is ``potential'') if $T_i = 1$, and $Y_i(0)$ is the analogous outcome for when $T_i = 0$. 
Then, $Y_i = T_iY_i(1) + (1-T_i)Y_i(0)$. 
We denote $N_t$ as the number of treated units in the dataset, and $N_c$ the number of control units. 
Our task is to estimate the \textit{average treatment effect on the treated (ATT)}, defined as
\begin{align*}
    \text{ATT} = \mathbb{E}[Y(1) - Y(0) \mid T=1].
\end{align*}
This quantity measures the treatment effect for patients under treatment, and is typically the primary interest of medical applications \citepMain{Ho2007att}. 
Here, covariates, such as age or BMI that are related to a treatment are of particular interest \citepMain{greifer2021att}. 
The ATT can be approximated by the \textit{sample average treatment effect on the treated (SATT)}, defined as
\begin{align*}
    \text{SATT} = \frac{1}{N_t}\sum_{i : T_i = 1} Y_i(1) - Y_i(0).
\end{align*}
We make the following standard assumptions :
\begin{itemize}
    \item Consistency: $\forall t, \ T_i = t \implies Y_i(t) = Y_i$.
    \item Ignorability: $Y_i(1),Y_i(0) \indep T_i | X_i$.
    \item Overlap: $\forall \B{x}, \ 0 < P(T_i=1|X_i=\B{x}) < 1$.
\end{itemize}
Consistency ensures that $Y_i(1)$ is the observed outcome $Y_i$ when $T_i = 1$. 
However, $Y_i(0)$ is not observed and must be estimated, for instance through matching.
In addition, the ATT can also be expressed using \textit{conditional average treatment effects} as
\begin{align}
    \text{ATT} = \mathbb{E}_X\Big[ \mathbb{E}[Y|T=1,X] - \mathbb{E}[Y|T=0,X] \Big| T=1 \Big],
    \label{eq:att_conditional}
\end{align}
which can be approximated by taking the sample mean over units as
\begin{align}
\label{eq:att_gt}
   \text{ATT} \approx \frac{1}{N_t}\sum_{i : T_i = 1} \mathbb{E}[Y_i|T_i=1,X_i] - \mathbb{E}[Y_i|T_i=0,X_i].
\end{align}
While we focus on the potential outcomes framework in this work, we note that an alternative is Pearl's framework of directed acyclic graphs (DAGs) and structural causal models (SCMs) \citepMain{pearl2009causality}. 

\subsection{Key Concepts}
\label{sec:Key concepts}

In general, a matching procedure generates weights $w_{ij}$ denoting the assignment of one or many control units $j$ to a treated unit $i$  \citepMain[Chapter 5]{MorganWinship2014matchingcha5}. 
Typically, matching only assigns few control units, i.e. for a treated unit $i$, there is a small number of control units $j$ such that $w_{ij} > 0$, and $w_{ij} = 0$, otherwise. 
This yields a new, weighted dataset $(w_i,X_i,T_i,Y_i) \sim P'$, where $w_i = 1$ for all treated units $i$ and $w_j = \sum_{i} (w_{ij} / \sum_{j} w_{ij}$) for control units $j$. %
The matching procedure serves two main goals. One is to estimate the ATT through the following estimator of the potential outcome $Y_i(0)$ :
\begin{align*}
\hat{Y}_i(0) = \frac{1}{\sum_{j: T_j = 0} w_{ij}}\sum_{j: T_j=0} w_{ij}Y_j.
\end{align*}
Another is to obtain \textit{balance} or, when it is not possible, reduce \textit{imbalance} in $P'$ compared to the original distribution $P$. 
Balance occurs when the distributions of covariates $X$ given $T=0$ on the one hand and $T=1$ on the other hand are equal. 
Perfect balance thus corresponds to zero imbalance, and is desirable because it eliminates confounding. In this ideal setting, the treatment effect can then be estimated as the difference between averaged outcomes in both distributions. 
In this sense, the two goals of treatment effect estimation and balance are related.
However, there is also a bias-variance trade-off at stake, as selecting fewer matching units will reduce imbalance and thus the expected treatment estimation error or ``bias'', at the cost of increased variance.

Formally, for a random variable $A$, we refer to the statement
\begin{align*}
    P(A|T=1) = P(A|T=0) 
\end{align*}
as ``balance in $A$'', and, for a function $D$ of two probability distributions, we refer to the quantity
\begin{align*}
    D(P(A|T=1),P(A|T=0)) 
\end{align*}
as ``$D$-imbalance in $A$''. 
When $D$ is a probability distance, e.g. total variation or Wasserstein distance, then a zero $D$-imbalance in $A$ implies balance in $A$. This is not true when $D$ is \textit{not} a probability distance, e.g. linear MMD.
Note to distinguish $D$ from the notation for a distance metric $d$ in Section~\ref{sec:Related work}.
We omit the mention of $D$ or $A$ when obvious from the context.
     
There are different ways to measure imbalance, such as a (standardised) difference in means~\citepMain{Austin2011psm}, integral probability metrics such as the Wasserstein distance, the maximum mean discrepancy (MMD) and the total variation (TV)~\citepMain{Sriperumbudur2012ipms, Kallus2020deepmatch}, or histogram-based $L_1$ distances~\citepMain{Iacus2012CEM}. 
Balance and imbalance can also apply to other variables than covariates, such as transformations of covariates \citepMain{johansson2016learning,Shalit2017estimating,iacus2011multivariate}.

\section{RELATED WORK}
\label{sec:Related work}
We now discuss existing work on matching and alternative approaches in causal inference that aim to reduce imbalance or estimate the ATT.
Most commonly, choosing matched control units $j$ is done through a nearest neighbours search among all control units $j$ according to some distance metric $d(.,.)$~\citepMain{stuart2010matching}. 
Nearest neighbour search can be performed with or without replacement, and additionally, one may enforce a caliper, i.e. a maximal distance between matches.
Alternatively, one might consider all matches simultaneously through an optimisation programme (optimal matching)~\citepMain{Rosenbaum1989optmatching}. 
The choice of the distance metric $d$ differs between common matching techniques:
\begin{itemize}[noitemsep,topsep=0pt,parsep=0pt,partopsep=0pt,leftmargin=*] 
    \item Exact matching~\citepMain{Rosenbaum1985constructing}: $d(X_i,X_j) = \infty$, if $X_i \neq X_j$, and $d(X_i,X_j) = 0$, otherwise.
    \item Coarsened exact matching \citepMain{Iacus2012CEM}: for a function $f$, $d(X_i,X_j) = \infty$, if $f(X_i) \neq f(X_j)$, and $d(X_i,X_j) = 0$, otherwise. 
    $f$ is typically an element-wise function, mapping to some (aggregated) value.
    \item Mahalanobis distance matching~\citepMain{stuart2010matching}: $d(X_i,X_j) = (X_i - X_j)^T\Sigma^{-1}(X_i - X_j)$, where $\Sigma$ is the estimated covariance matrix of the control dataset in the case of ATT estimation.
    \item Propensity score matching~\citepMain{Austin2011psm}: $d(X_i,X_j) = \vert \hat{e}(X_j) - \hat{e}(X_i) \vert$ where $\hat{e}(\B{x})$ is an estimate of the propensity score $e(\B{x}) := P(T=1|X=\B{x})$. 
    This method is based on the property that $X \indep T \mid e(X)$. 
    We provide more details on implications of this property in Section \ref{sec:Balancing scores with regards to the latent space}.
\end{itemize}
Other than coarsened exact matching for which the weights have a different formulation, these methods set $w_{ij} = 1$ for matched units $i$ and $j$, and $w_{ij} = 0$, otherwise.

All the above matching methods suffer from the \textit{curse of dimensionality}, rendering them impractical in high-dimensional datasets. 
In general, theoretical results on nearest neighbour matching, to which the above techniques belong, show that the bias of the resulting ATT estimator grows with the data dimension $D$ at a rate $\mathcal{O}(N^{-r/D})$, where $N$ is the sample size and $r \geq 1$ is a constant~\citepMain{Abadie2006highdimmatching}. 
More precisely, exact matching and coarsened exact matching remove more and more control items as the number of covariates increases. 
Further, matching based on the Mahalanobis distance performs poorly in high dimensions, likely because all covariate interactions are assumed to be equally important~\citepMain{stuart2010matching}.

In the literature, the preferred method for high dimensions is propensity score matching. 
However, compression into a single dimension can lead to matches with very different characteristics in the original covariate space, as for a fixed compression, there is no other information used to choose matches: matching is then done at random. 
This applies to all compressions of covariates, however as the propensity score $p(T=1|X)$ is the coarsest compression which can be used for matching (see Section \ref{sec:Balancing scores with regards to the latent space}), with the least information from $X$, it is most prone to actually matching at random. 
This can increase imbalance and consequently bias~\citepMain{King2019nopsm}.
Other than propensity score methods, approaches for matching in high dimensions include penalised regression techniques such as LASSO to perform variable selection before matching \citepMain{schneeweiss2009high, belloni2013, farrell2015robust}, sufficient dimension reduction \citepMain{Luo2020sdr, cheng2020sufficient}, and distance metric learning \citepMain{Li2016MatchingVD, wang2021flame}. 

An alternative to matching is \textit{weighting}, where weights $w_j$ in the weighted dataset are directly estimated, generalising the problem formulation of matching~\citepMain{Kallus2020gom}.  
Examples include leveraging the estimated propensity score for inverse probability weighting~\citepMain{Horvitz1952ipw} or learning weights directly~\citepMain{Kallus2020deepmatch}. 
A second alternative to matching is \textit{outcome regression}.
These methods estimate the quantity $\mathbb{E}[Y|T=t, X=\B{x}]$ through a regressor $Q(t,\B{x})$ that can be fitted through various methods related to linear regression~\citepMain{imbens_rubin_2015}, tree models~\citepMain{Athey2019generalized}, or neural networks~\citepMain{Shi2019adapting, Shalit2017estimating}. 
Combining weighting through the propensity score estimate and outcome regression leads to the popular doubly robust methods, such as the augmented inverse probability weighted (AIPW) method \citepMain{Robins1994aipw}. 
Recent efforts have been made to recategorise and benchmark outcome regression and doubly robust methods~\citepMain{Curth2021nonparametric}.

\section{NEURAL SCORE MATCHING}
\label{sec:Variational autoencoders for matching}

In this section, we present theoretical results that motivate the use of neural networks to obtain non-trivial, multivariate balancing scores.
This approach aims to address the curse of dimensionality problem, as outlined in the previous section. 
In addition, some of these results have wider applicability to other models than neural networks.

\subsection{Balancing Scores}
\label{sec:Balancing scores with regards to the latent space}

We start by defining and analysing the use of \textit{balancing scores}. 
This notion also motivated propensity score matching~\citepMain{Rosenbaum1983ps}. 

\begin{definition} \label{def:balancing_scores}
A \emph{balancing score} is a function $b$ of $X$ such that $X \indep T \mid b(X)$.
\end{definition}

As a consequence, for a fixed value $\B{\beta}$ of $b(X)$, it holds that
\begin{align*}
    P(X \mid b(X)=\B{\beta}, T=1) = P(X \mid b(X) = \B{\beta}, T=0),
\end{align*}
i.e. the treatment and control distributions in the covariate space are equal for any fixed realisation of $b(X)$. 
Notably, it is possible to show that average treatment effects can be estimated by conditioning on $b(X)$ instead of $X$ in Equation~\eqref{eq:att_conditional} \citepMain{Rosenbaum1983ps}.

We can further connect (im)balance in $b(X)$ to (im)balance in $X$. 
The following Proposition shows that $TV$-imbalance in $X$ is equal to $TV$-imbalance in $b(X)$, where $TV$ is the total variation distance.

\begin{theorem} \label{th:tv_eq_bX}
Let $b$ be a function such that $b(X)$ is a balancing score. 
Then,
\begin{align*}
&TV\left(P(X \mid T=1),P(X \mid T=0)\right) \\
&= TV\left(P(b(X) \mid T=1),P(b(X) \mid T=0)\right).
\end{align*}
\textit{Proof: }\hspace{3mm}See Appendix~\ref{app:balance_bX_X}.\hspace{1mm}$\square$
\end{theorem}

This allows us to potentially use lower-dimensional balancing scores $b(X)$ instead of high-dimensional covariates to achieve balance in $X$, as the following corollary shows that balance in $b(X)$ ensures balance in $X$ :

\begin{corollary} \label{th:balance_X}
Under the same conditions as Proposition~\ref{th:tv_eq_bX}, 
\begin{align*}
&P(b(X) \mid T=1) = P(b(X) \mid T=0) \\
&\implies P(X \mid T=1) = P(X \mid T=0).
\end{align*}
\textit{Proof: }\hspace{3mm}See Appendix~\ref{app:balance_bX_X}.\hspace{1mm}$\square$
\end{corollary}

Matching on a given balancing score $b(X)$ is commonly used to reduce imbalance in $b(X)$, with the aim of consequently reducing imbalance in $X$. Proposition \ref{th:tv_eq_bX} shows that a lower $TV$-imbalance in $b(X)$ will also mean a lower $TV$-imbalance in $X$, but only if $b(X)$ remains a balancing score in the post-matching distribution $P'$. Thankfully, the following Proposition shows that $b(X)$ remains a balancing score after matching.

\begin{theorem} \label{th:phew}
Let $b$ be a function such that $b(X)$ is a balancing score, $P'$ be a distribution obtained from matching every treated unit with control units using $b(X)$ only. Then $b(X)$ is also a balancing score in $P'$.
\newline \par\vspace{-.3cm} \noindent \textit{Proof: }\hspace{3mm}See Appendix~\ref{app:balance_bX_X}.\hspace{1mm}$\square$
\end{theorem}
Thus, all further theoretical results involving balancing scores in the original distribution will also be valid in the matched distribution. An important question left open at this point is how to find such a function $b$ such that $b(X)$ is a balancing score.

Leveraging theoretical results in \citepMain{Rosenbaum1983ps}, balancing scores can be linked to the propensity score $e(X) = P(T=1 \mid X)$. 

\begin{theorem} \label{th:ps_balancing}
A function $b(X)$ is a balancing score, if and only if $b(X)$ can be mapped deterministically to the propensity score $e(X)$ through a function $f$, i.e. 
\begin{align*}
e(X) = f(b(X)).
\end{align*}
\textit{Proof: }\hspace{2mm}See~\citepMain[Thm. 2]{Rosenbaum1983ps}.\hspace{1mm}$\square$
\end{theorem}

It follows from Proposition~\ref{th:ps_balancing} that $e(X)$ is itself a balancing score for the identity map. 
When this identity does not hold, $b(X)$ is said to be ``finer'' than $e(X)$, and conversely, $e(X)$ is ``coarser'' than $b(X)$. 
As noted in~\citepMain{Rosenbaum1983ps}, $X$ is the finest balancing score, containing the most information; $e(X)$ is the coarsest balancing score, containing the least information; and any other $b(X)$ such that $e(X) = f(b(X))$ lies between the two. 
Choosing the degree of coarseness via multi-dimensional balancing scores to achieve optimal matching results rather than assuming a one-dimensional balancing score (i.e. the propensity score) is what we exploit in our method which we introduce in the following.

\subsection{Introducing Neural Score Matching}

Previous work has largely focused on the use of the propensity score $e(X)$ as a balancing score, and relatively little attention has been paid to non-trivial balancing scores that are neither $X$ nor $e(X)$. 
Neural networks provide a natural mechanism by which to construct such balancing scores: 
fundamentally, a multi-layer neural network is a composition of functions $f_1, f_2, \dots, f_L$. 
Let us for a moment assume this network (perfectly) estimates the propensity score, i.e. $\hat{e}(X) = f_L \circ f_{L-1} \circ ... \circ f_1(X) = e(X)$.  
Then, by Proposition~\ref{th:ps_balancing}, this provides us with $L+1$ balancing scores ($X$, the $L-1$ intermediate hidden representations and the estimated propensity score) that are coarser and coarser with increasing ``depth'' of the composition. 
We note that instead of neural networks parameterising $f_1, f_2, \dots, f_L$, one may consider other hierarchical models.
We formalise this general principle, which we call \textit{neural score matching}, in the following Proposition:

\begin{theorem} \label{th:all_balancing}
Assume that $e(X) = f_L \circ f_{L-1} \circ \dots \circ f_1(X)$ for some functions $f_1, \dots, f_L$. Define $b_0(X) := X$ and $b_l(X) = f_l \circ f_{l-1} \circ \dots \circ f_1(X)$ for $l = 1,\dots,L$. Then, every $b_l(X)$ is a balancing score, and for any $l < L$, $b_{l+1}(X)$ is coarser than $b_l(X)$. 
\newline \par\vspace{-.3cm} \noindent  \textit{Proof: }\hspace{3mm}See Appendix~\ref{app:nsm}.\hspace{1mm}$\square$
\end{theorem}

Using this Proposition, we can now connect these balancing scores to our goal of achieving balance in $X$:

\begin{corollary} \label{cor:all_balancing_matching}
Under the same conditions and notation as Proposition~\ref{th:all_balancing}, for any $l, l' = 0, \ldots, L$,
\begin{align*}
&TV\left(P(b_{l}(X) \mid T=1),P(b_{l}(X) \mid T=0)\right) \\
&= TV\left(P(b_{l'}(X) \mid T=1),P(b_{l'}(X) \mid T=0)\right),
 \end{align*}
 and balance in $b_{l'}(X)$ is equivalent to balance in $b_l(X)$.
\newline \par\vspace{-.3cm} \noindent \textit{Proof: }\hspace{3mm}See Appendix~\ref{app:nsm}.\hspace{1mm}$\square$
\end{corollary}

This Proposition gives us a choice of balancing scores with varying degree of coarseness which we can use for matching. 
Note that achieving balance in \textit{any} of the scores will yield balance in \textit{all} of them, and particularly in $X = b_0(X)$. 
On the other hand, perfect balance is difficult to attain, but we can still aim to achieve the lowest imbalance possible. Importantly, although imbalance is identical for two given balancing scores in the same hierarchical propensity score model \textit{when the distribution is fixed}, matching on these two balancing scores will in general result in different distributions and consequently different imbalances.

Thus, if we can compute $TV$-imbalances, the Proposition ensures that selecting the balancing score and matching procedure with the lowest resulting $TV$-imbalance will also reach the lowest $TV$-imbalance in covariate distributions $X$ among the candidate balancing scores and matching procedures.

It is important to note that Proposition \ref{th:all_balancing}, Corollary \ref{cor:all_balancing_matching} and the following theoretical results all assume that $\hat{e}(X) = e(X)$, i.e. a well-calibrated propensity score model, or at least that the obtained scores are indeed balancing scores.
In Section \ref{sec:approx_bound}, we will relax this assumption and provide theoretical bounds when scores violate the balancing score assumption from Definition \ref{def:balancing_scores}.

In practice, however, the total variation distance is not suitable for this purpose of balancing score comparison due to the difficulties with estimating it in finite samples~\citepMain{Kallus2020deepmatch}. 
We provide results with alternative metrics which overcome this issue in Section \ref{sec:wass_and_mmd}

\subsection{Bounds With Estimable Integral Probability Metrics}
\label{sec:wass_and_mmd}

We start with a general inequality that shows that any imbalance in $X$ measured using an integral probability metric (IPM) is also upper-bounded by such an imbalance in $b(X)$.

\begin{theorem}  \label{th:ipm}
Let $\mathcal{F}$ be a set of functions of $X$. For any function $b$ of $X$, define
\begin{align*}
\mathcal{F}_b := \left\{ \ \B{\beta} \mapsto \mathbb{E}[f(X) \mid b(X) = \B{\beta}], \ \ f \in \mathcal{F} \right\}.
\end{align*}
Then, for any balancing score $b(X)$ and any set  $\mathcal{G}$ of functions on the image set of $b$ such that $\mathcal{F}_b \subseteq \mathcal{G}$,
\begin{align*}
    &\text{IPM}_{\mathcal{F}}\left(P(X \mid T=1), P(X \mid T=0)\right)\\
    &\leq \text{IPM}_{\mathcal{G}}\left(P(b(X) \mid T=1), P(b(X) \mid T=0)\right) 
\end{align*}
with equality when $\mathcal{G} = \mathcal{F}_b$.
\end{theorem}

As a result, any measure of imbalance of original covariates based on an IPM, including using popular ones such as the linear MMD or the Wasserstein distance, can be controlled using another measure of imbalance depending on an IPM. Thus, as in Corollary \ref{cor:all_balancing_matching}, we expect to reduce any IPM-imbalance in $X$ when reducing another IPM-imbalance in $b(X)$, further justifying matching on $b(X)$ as an alternative to matching on $X$ when the measure of interest for imbalance in $X$ is an IPM. Further, if we had access to the $\text{IPM}_{\mathcal{F}_b}$-imbalance in $b(X)$, we could again use it to select the appropriate balancing score, as for the total variation distance. One caveat is that it is precisely unclear \emph{which} IPM-imbalance in $b(X)$ is suitable in Proposition \ref{th:ipm} as the class $\mathcal{F}_b$ is non-trivial due to the conditional expectation in $b(X)$, even for common base classes $\mathcal{F}$ such as linear or Lipschitz functions. Thus, the question remains whether we can bound the IPM-imbalance of $X$ using a \emph{computable} IPM-imbalance.

To solve this, we consider a \textit{linear} balancing score $b(X)$, meaning that $b$ is a linear function.
For example, this can be realised by considering the first layer of a neural network before applying an activation function.
In this simple case, and under strong assumptions on the distribution of $X$, we can leverage popular integral probability metrics which \textit{can} be estimated with finite samples, in contrast to the total variation distance.

\begin{theorem}  \label{th:wass_and_mmd}
Let $b$ be a function such that $\forall \B{x}, \ b(\B{x}) = W\B{x}$ for some matrix $W$ and $b(X)$ is a balancing score.
Let $||.||$ be the Euclidean norm on any vector space, and $|||.|||$ be a norm\footnote{Examples include the operator norm or the Euclidean norm.} on any matrix space such that $\forall \B{x}, A, ||A\B{x}|| \leq |||A||| \cdot ||\B{x}||$. 
Further, let $A^+$ be the Moore-Penrose pseudo-inverse of $A$, $\wass$ be the Wasserstein distance, $\mmd$ be the linear MMD\footnote{Note that these theoretical results also hold when $b(X)$ has a bias term.}. Let $W^+_\Sigma := \Sigma W^T (W\Sigma W^T)^+$. If $X$ is elliptical with covariance matrix $\Sigma$ then
\begin{align*}
&\frac{1}{|||W|||} \cdot \mmd\left(P(b(X) \mid T=1),P(b(X) \mid T=0)\right) \\
&\leq \mmd\left(P(X \mid T=1),P(X \mid T=0)\right) \\
&\leq |||W^+_\Sigma|||\cdot \mmd\left(P(b(X) \mid T=1),P(b(X) \mid T=0)\right)
\end{align*}
If $X$ is Gaussian with positive-definite covariance matrix $\Sigma$ and $W$ has full row rank then
\begin{align*}
&\frac{1}{|||W|||} \cdot \wass\left(P(b(X) \mid T=1),P(b(X) \mid T=0)\right) \\
&\leq \wass\left(P(X \mid T=1),P(X \mid T=0)\right) \\
&\leq |||W^+_\Sigma||| \cdot \wass\left(P(b(X) \mid T=1),P(b(X) \mid T=0)\right).
\end{align*}
\end{theorem}

This Proposition provides lower- and upper-bounds (in contrast to Proposition \ref{th:tv_eq_bX}) for the Wasserstein- or linear MMD-imbalance in $X$ which depend linearly on the corresponding imbalance in $b(X)$. 

One could exploit these bounds by computing them for different balancing scores and choose the one with the lowest (lower or upper) bounds of the interval, or the narrowest bounds. 
One might also perform a type of optimal matching minimising the Wasserstein or linear MMD imbalance in $b(X)$. 
However, it is important to point out that these bounds may be wide depending on the singular values of $W$. For example, assume $\Sigma = I$, then $W^+_\Sigma = W^+$. 
When using the operator norm and denoting $\sigma_\text{min}(W)$ and $\sigma_\text{max}(W)$ as the minimal and maximal non-zero singular values of $W$  \footnote{This assumes $W \neq 0$, i.e. we do not have balance in $X$.}, respectively, we have $\frac{1}{|||W|||} = \frac{1}{\sigma_\text{max}(W)}$ and $|||W^+||| = \frac{1}{\sigma_\text{min}(W)}$. As a consequence, values within the bounds can vary by a factor of $\frac{ \sigma_\text{max}(W)}{ \sigma_\text{min}(W)}$. Further, the strong assumptions on the distribution of $X$ might not hold in practice, especially in the post-matching distribution.

In addition, the imbalance in $b(X)$ might also help speed up computations.  In Appendix \ref{app:complexity_wass}, we show how the computational complexity of the estimators of the Wasserstein distance can be reduced on a lower-dimensional space.

From the insights of Proposition \ref{th:wass_and_mmd}, we only use the first layer of a neural network for the purpose of matching; the other layers serve to achieve a better fit of the propensity score model.

\subsection{Bounds For Non-Balancing Scores}
\label{sec:approx_bound}

As mentioned above, a requirement for applying the above Propositions within the context of hidden representations of a neural network is that either the estimated propensity score of said network equals the true propensity score, or more generally, every learned function $b$ is indeed a balancing score. 
When this is not the case, as the next Proposition shows, we can still bound the imbalance in $X$ in terms of the imbalance in $b(X)$ and some quantification of ``how much'' the assumption $X \indep T | b(X)$ is violated.

\begin{theorem} \label{th:approx}
Let
\begin{align*}
\mathcal{E}^D_{t,b}(\B{\beta}) := D\Big( P\big(X | b(X) = \B{\beta} , T=t\big), P\big(X |  b(X) = \B{\beta} \big) \Big)
\end{align*}
where $D$ is a probability discrepancy measure, $b$ is a function of $X$, $t \in \{0,1\}$ is a realisation of $T$, $\B{\beta}$ is a realisation of $b(X)$. For any function $b$,
    \begin{align*}
&TV\Big( P\big(b(X) | T=1\big), P\big(b(X) | T=0\big) \Big) \\
&\leq TV\Big( P\big(X | T=1\big), P\big(X | T=0\big) \Big) \\
& \leq TV\Big( P\big(b(X) | T=1\big), P\big(b(X) | T=0\big) \Big) \\
& + \mathbb{E}\big[\mathcal{E}^{TV}_{1,b}\big(b(X)\big) | T=1\big] + \mathbb{E}\big[\mathcal{E}^{TV}_{0,b}\big(b(X)\big) | T=0\big]
\end{align*}
and, using the notations of Proposition \ref{th:ipm},
\begin{align*}
    &\text{IPM}_{\mathcal{F}}\left(P(X \mid T=1), P(X \mid T=0)\right)\\
    &\leq \text{IPM}_{\mathcal{F}_b}\left(P(b(X) \mid T=1), P(b(X) \mid T=0)\right) \\
    &+ \mathbb{E}\big[\mathcal{E}^{\text{IPM}_{\mathcal{F}}}_{1,b}\big(b(X)\big) | T=1\big] + \mathbb{E}\big[\mathcal{E}^{\text{IPM}_{\mathcal{F}}}_{0,b}\big(b(X)\big) | T=0\big].
\end{align*}
For a linear function  $b(x) = Wx$, if $X$ is elliptical with covariance matrix $\Sigma$, then
\begin{align*}
&\frac{1}{|||W|||} \cdot
 \mmd\Big( P\big(b(X) | T=1\big), P\big(b(X) | T=0\big) \Big) \\
&\leq \mmd\Big( P\big(X | T=1\big), P\big(X | T=0\big) \Big) \\
& \leq |||W^+_\Sigma||| \cdot
 \mmd\Big( P\big(b(X) | T=1\big), P\big(b(X) | T=0\big) \Big) \\
 &
 + \mathbb{E}\big[\mathcal{E}^{\mmd}_{1,b}\big(b(X)\big) | T=1\big] + \mathbb{E}\big[\mathcal{E}^{\mmd}_{0,b}\big(b(X)\big) | T=0\big]
\end{align*}
and if $X$ is Gaussian with positive-definite covariance matrix $\Sigma$ while $W$ has full row rank, then
\begin{align*}
&\frac{1}{|||W|||} \cdot
 \wass\Big( P\big(b(X) | T=1\big), P\big(b(X) | T=0\big) \Big) \\
& \leq \wass\Big( P\big(X | T=1\big), P\big(X | T=0\big) \Big) \\
& \leq |||W^+_\Sigma||| \cdot
 \wass\Big( P\big(b(X) | T=1\big), P\big(b(X) | T=0\big) \Big) \\
 &
 + \mathbb{E}\big[\mathcal{E}^{\wass}_{1,b}\big(b(X)\big) | T=1\big] + \mathbb{E}\big[\mathcal{E}^{\wass}_{0,b}\big(b(X)\big) | T=0\big].
\end{align*}
\end{theorem}

Unlike the calibration error, i.e. the mean difference between true and predicted propensity scores, the extra balancing error term in the Proposition does not rely on access to the true propensity score. Therefore, it could be computed and used to obtain an upper bound of covariate imbalance in any dataset. In practice, however, it might be challenging to estimate as it relies on conditional expectations for which few samples may be available.

\section{EXPERIMENTS}
\label{sec:Experiments}

We now evaluate neural score matching on two semi-synthetic datasets and benchmark it against other matching methods. 
We provide code to implement neural score matching and reproduce the main results at \textcolor{blue}{\href{https://github.com/oscarclivio/neuralscorematching}{\url{https://github.com/oscarclivio/neuralscorematching}}}.

\subsection{Experimental Setup}

Our general procedure for matching and in particular neural score matching follows two stages: in the first stage, we learn a model to obtain some representation or score $s$ from datapoints. 
In the second stage, we perform matching on these scores using the Euclidean distance\footnote{We use the Euclidean distance as the Mahalanobis distance was prohibitively slow to compute in high-dimensional and large sample settings.} $d(s_i,s_j) = ||s_i - s_j||_2$.
We use nearest neighbour matching with replacement using one neighbour. 

To perform neural score matching, we train a neural network predicting treatment assignment from covariates, with the final one-dimensional layer being an estimator of the propensity score. Training is performed using a standard binary cross-entropy loss.
The neural network has the following architecture: 
one low-dimensional layer with 5 hidden units, two layers with 100 units and one final 1-dimensional layer. 
We use the leaky ReLU activation function in all layers except the last one where we use the sigmoid function. 
When using the hidden representation in the first layer before applying the activation function as a %
score, we refer to the resulting method as \texttt{NN Layer 1}. Notably, if it is indeed a balancing score, it meets the assumptions of Proposition \ref{th:wass_and_mmd}. 
From the insights of this Proposition, we choose to focus on one single multivariate layer for matching, and dedicate other layers to model fitting (with corresponding high dimensions as given above).
The final activation of the network estimates the propensity score and is also used for matching as a balancing score. 
We refer to it as \texttt{NN PS}. 

We benchmark these scores obtained by the neural network against other scores, namely $X$ (\texttt{X}) and a five-dimensional PCA reduction of $X$ (\texttt{PCA}).
We also benchmark against a logistic regression estimating the propensity score given $X$ or PCA features, which we refer to as \texttt{LogReg PS} and \texttt{PCA + LogReg PS}, respectively. 
In addition, we consider matching uniformly at random (\texttt{Random matching}) and leaving the treatment and control datasets unchanged w.r.t. balance by not matching at all (\texttt{No Matching}). 
All methods were evaluated using 10 different training random seeds.

We use variants of two standard datasets for treatment effect estimation: \textit{ACIC 2016} and \textit{News}. Both datasets have a large number of covariates (82 and 3477, respectively), rendering them challenging for standard matching techniques. 
They are both semi-synthetic: the covariates come from real-world studies, while the treatments and outcomes were generated through a data generating process. 
For every dataset, we will average results over different draws of the data generating process (100 for ACIC 2016 and 50 for News).
Early stopping was used on News. 
Results on a third dataset, IHDP, are presented in Appendix \ref{app:ihdp}.

To evaluate the methods, we report three metrics: 
\textit{calibration error}, defined as the mean absolute difference between the estimated and true propensity score, \textit{ATT error}, defined as the absolute difference between the ATT estimated by the method and a ground-truth ATT, and \textit{sample imbalance} $\hat{I}$, defined as the squared Euclidean distance between sample means of covariates of treated and control groups from the dataset $\mathcal{D}'$ obtained from the original dataset $\mathcal{D}$ after matching. 
To reliably assess the performance of the methods under investigation, we average and present standard deviations over the repeated draws of the data generating processes and additionally over the different random seeds for model fitting/training. 

We refer to Appendix~\ref{app:Implementation Details} for further details about implementation and experimental setup.

\subsection{Experimental Results}

In this section, we present our experimental results as Tables (and refer to Appendix~\ref{app:boxplots} for their visualisation as boxplots).

\subsubsection{ACIC 2016}

Results for the different matching methods under consideration are presented in Table \ref{tab:acic2016}. 
Propensity score models for the two dimensionality reduction methods (\texttt{NN Layer 1} and \texttt{PCA}) have better calibration than the standard logistic-regression propensity score (\texttt{LogReg PS}), with a slight advantage for \texttt{NN PS}. 
The relevance of using a multivariate score is demonstrated: on ATT errors and imbalances, \texttt{NN Layer 1} most often outperforms \texttt{NN PS}, and all other methods except:
\begin{itemize}
    \item Logistic regression propensity score (\texttt{LogReg PS}) on in-sample metrics. It is possible that the dimensionality remains sufficiently low for this method to handle (unlike News, see next section). However, the method might also overfit, as shown by the hold-out performance.
    \item \texttt{No Matching} and \texttt{PCA} on hold-out imbalances. Neural scores might need better generalisation as they increase imbalance compared to the original dataset, unlike \texttt{PCA}. 
    Other methods also increase imbalance, as expected.
\end{itemize}

\begin{table}[ht]
\caption{\label{tab:acic2016} Results on the ACIC2016 dataset.}
\begin{center}
\setlength{\tabcolsep}{2pt}
\begin{tabular}{lcc}
\toprule
 Calibration errors & In-Sample & Hold-Out \\
\midrule
\texttt{NN PS} (ours) & 0.055$\pm$0.000 & 0.055$\pm$0.000 \\
\texttt{LogReg PS} & 0.067$\pm$0.000 & 0.069$\pm$0.000 \\
\texttt{PCA + LogReg PS} & 0.058$\pm$0.001 & 0.058$\pm$0.001 \\
\midrule
ATT errors & In-Sample & Hold-Out \\
\midrule
\texttt{NN Layer 1} (ours) & 0.707$\pm$0.012 & 0.918$\pm$0.018 \\
\texttt{NN PS} (ours) & 0.735$\pm$0.012 & 1.008$\pm$0.019 \\
\texttt{X} & 0.848$\pm$0.018 & 0.990$\pm$0.019 \\
\texttt{Random matching} & 1.209$\pm$0.019 & 1.301$\pm$0.023 \\
\texttt{LogReg PS} & 0.678$\pm$0.012 & 1.036$\pm$0.018 \\
\texttt{PCA} & 0.927$\pm$0.016 & 1.007$\pm$0.020 \\
\texttt{PCA + LogReg PS} & 0.962$\pm$0.016 & 1.097$\pm$0.021 \\
\midrule
Sample imbalance & In-Sample & Hold-Out \\
\midrule
\texttt{NN Layer 1} (ours) & 0.107$\pm$0.001 & 0.422$\pm$0.003 \\
\texttt{NN PS} (ours) & 0.105$\pm$0.001 & 0.498$\pm$0.004 \\
\texttt{X} & 0.438$\pm$0.002 & 0.739$\pm$0.004 \\
\texttt{Random matching} & 0.232$\pm$0.003 & 0.558$\pm$0.006 \\
\texttt{LogReg PS} & 0.056$\pm$0.001 & 0.511$\pm$0.004 \\
\texttt{PCA} & 0.117$\pm$0.001 & 0.342$\pm$0.003 \\
\texttt{PCA + LogReg PS} & 0.134$\pm$0.001 & 0.488$\pm$0.004 \\
\texttt{No Matching} & 0.192$\pm$0.003 & 0.396$\pm$0.006 \\
\bottomrule
\end{tabular}
\end{center}
\end{table}

\subsubsection{News}

Results for the News dataset are presented in Table \ref{news}. 
Multivariate dimensionality-reduced scores (\texttt{NN Layer 1} and \texttt{PCA}) generally outperform their respective propensity scores (except \texttt{NN Layer 1} and \texttt{NN PS} having similar performance on ATT errors), as well as \texttt{Random matching}, \texttt{X} and \texttt{LogReg PS}. 
The two latter have particularly high ATT errors and imbalances, even compared to \texttt{Random matching}. 
This shows that multivariate, but lower-dimensional scores can improve matching on high-dimensional datasets. 
The performance is more balanced between \texttt{PCA} and \texttt{NN Layer 1}: \texttt{PCA} is better on imbalances, \texttt{NN Layer 1} on in-sample ATT errors, and their hold-out ATT errors are not significantly different according to standard errors.

\begin{table}[ht]
\centering
\caption{\label{news} Results on the News dataset.}
\setlength{\tabcolsep}{2pt}
\begin{tabular}{lcc}
\toprule
ATT errors & In-Sample & Hold-Out \\
\midrule
\texttt{NN Layer 1} (ours) & 0.071$\pm$0.002 & 0.106$\pm$0.004 \\
\texttt{NN PS} (ours) & 0.073$\pm$0.002 & 0.105$\pm$0.004 \\
\texttt{X} & 0.510$\pm$0.015 & 0.765$\pm$0.024 \\
\texttt{Random matching} & 0.100$\pm$0.003 & 0.114$\pm$0.004 \\
\texttt{LogReg PS} & 1.460$\pm$0.052 & 0.505$\pm$0.020 \\
\texttt{PCA} & 0.080$\pm$0.002 & 0.103$\pm$0.003 \\
\texttt{PCA + LogReg PS} & 0.095$\pm$0.003 & 0.100$\pm$0.003 \\
\midrule
Sample imbalance & In-Sample & Hold-Out \\
\midrule 
\texttt{NN Layer 1} (ours) & 1.518$\pm$0.022 & 3.886$\pm$0.045 \\
\texttt{NN PS} (ours) & 2.104$\pm$0.035 & 5.105$\pm$0.079 \\
\texttt{X} & 12.531$\pm$0.032 & 18.178$\pm$0.052 \\
\texttt{Random matching} & 2.121$\pm$0.041 & 4.581$\pm$0.043 \\
\texttt{LogReg PS} & 371.070$\pm$36.672 & 131.192$\pm$4.682 \\
\texttt{PCA} & 1.097$\pm$0.013 & 3.608$\pm$0.030 \\
\texttt{PCA + LogReg PS} & 1.444$\pm$0.017 & 4.600$\pm$0.046 \\
\texttt{No Matching} & 1.844$\pm$0.040 & 3.432$\pm$0.038 \\
\bottomrule
\end{tabular}
\end{table}

\section{DISCUSSION AND CONCLUSION}
\label{sec:Conclusion}

In this work, we have provided novel theoretical results motivating neural score matching: using neural networks to obtain balancing scores which can be readily used for matching. 
In contrast to lower-dimensional representations obtained from classical dimensionality reduction techniques (e.g. PCA), our method estimates lower-dimensional balancing scores as defined in Proposition \ref{th:ps_balancing}, which can be mapped back to the propensity score ``for free'' due to the inherent compositionality of neural networks, allowing more flexibility in choosing the degree of coarseness. 
This applies only if the model is correctly specified, however. Proposition \ref{th:approx} paves the way to rigorous analysis of situations when the constraint is violated. %
We found that in popular semi-synthetic datasets, neural score matching is competitive against other matching methods. 
In addition, our results indicate the general utility of dimensionality reduction techniques for matching in causal inference.
This leads the way towards learning suitable representations for matching which might be useful for downstream tasks to gain scientific insight, notably in areas where the use of neural networks is ubiquitous, such as medical imaging \citepMain{Zhou2020imaging}, text classification \citepMain{Minaee2021text} and audio processing \citepMain{Purwins2019audio}.

Our work has the following two limitations: %
1) It is difficult to properly specify and train neural networks for the task of matching. 
In particular, there is a trade-off between finding low-dimensional balancing scores, which implies low-dimensional hidden layers, and fitting the propensity score model, which implies wide hidden layers not suitable for matching. 
We also did not find hyperparameters that performed consistently better than others across all datasets, nor a correlation between matching performance and hold-out loss. 
More complex architectures than our experimental setup and a deeper understanding of the hyperparameter space should be explored. 
2) Most of our theoretical results assume the propensity score model is correct, or, more generally speaking, that the hidden layers are indeed balancing scores. 
Most often, neither is true.
Proposition \ref{th:approx} is a first step towards theoretical guarantees for scores that are not perfectly balancing.

Future work will investigate the following ideas: 
1) As outlined earlier, our propensity score model might be miscalibrated and the balancing scores might not perfectly balance covariates.
Empirically measuring calibration error and the violation of the balancing score property via Proposition \ref{th:approx}, we aim at using this to inform model training and hence improve performance.
2) We plan to extend the relatively simple setup of neural score matching as presented here to, for instance, using multiple intermediate balancing scores. 
This entails further questions, such as how to choose the degree of coarseness of the balancing scores, which might be assessed via empirical out-of-sample evaluation, and where to best place layers with few hidden units that are suited for matching.
3) We aim to develop a form of optimal matching which uses more general bounds of $\wass$- or $\mmd$-imbalance in $b(X)$ than those of Proposition \ref{th:wass_and_mmd}, and use them directly in a loss function, which in turn should reduce imbalance in $X$.
4) Our obtained balancing scores might enable the use of coarsened exact matching (CEM), offering the possibility to pre-specify the desired level of imbalance before matching~\citepMain{Iacus2012CEM}.
5) We aim to explore more in depth how intermediate balancing scores compare to propensity scores, e.g. by visualising how their spaces capture features of the covariate space. 
We also expect these intermediate balancing scores to be preferable to propensity scores for CATE estimation as they provide less coarse representations of covariates.

\newpage

\subsubsection*{Acknowledgements}
O.C. is supported by the EPSRC Centre for Doctoral Training in Modern Statistics and Statistical Machine Learning (EP/S023151/1) and Novo Nordisk. 
F.F. acknowledges the receipt of a studentship award from the Health Data Research UK-The Alan Turing Institute Wellcome PhD Programme in Health Data Science (Grant Ref: 218529/Z/19/Z).
B.L. was supported by the UK Engineering and Physical Sciences Research Council through the Bayes4Health programme (grant number EP/R018561/1) and gratefully acknowledges funding from Jesus College, Oxford. 
C.H. acknowledges support from the Medical Research Council Programme Leaders award MC\_UP\_A390\_1107, The Alan Turing Institute, Health Data Research, U.K., and the U.K. Engineering and Physical Sciences Research Council through the Bayes4Health programme grant.

We would like to thank the anonymous reviewers for helpful feedback.

\bibliographystyleMain{apalike}
\bibliographyMain{references.bib}

\renewcommand{\theequation}{S\arabic{equation}}
\renewcommand{\thefigure}{S\arabic{figure}}
\renewcommand{\bibnumfmt}[1]{[S#1]}
\renewcommand{\citenumfont}[1]{S#1}

\onecolumn
\aistatstitle{Neural Score Matching for High-Dimensional Causal Inference: Appendices}
\appendix
\thispagestyle{empty}

\section{PROOFS OF THEORETICAL RESULTS}

\subsection{Balance on $b(X)$ and $X$}
\label{app:balance_bX_X}

\begin{reptheorem}{th:tv_eq_bX}
Let $b$ be a function such that $b(X)$ is a balancing score. 
Then,
\begin{align*}
&TV\left(P(X \mid T=1),P(X \mid T=0)\right) \\
&= TV\left(P(b(X) \mid T=1),P(b(X) \mid T=0)\right).
\end{align*}
\end{reptheorem}

\textit{Proof:}
\begin{itemize}
    \item First, let us note that for any random variable $V$, and by definition of the total variation distance:
    \begin{align*}
        TV(P(V|T=1), P(V|T=0)) &= \sup_{||f||_{L^\infty} \leq 1} \big|\mathbb{E}[f(V)|T=1] - \mathbb{E}[f(V)|T=0] \big| \\
        &= \text{IPM}_{\{f: \ ||f||_{L^\infty} \leq 1 \}}( P(V|T=1), P(V|T=0) ) ,
    \end{align*}
    where IPM is defined in Equation \ref{eq:ipm}, $||\cdot||_{L^\infty}$ is the uniform norm, and f is a function.
    \item For any function $f$ on the $\mathcal{B}$ space (i.e. the image space of $b(X)$) such that $||f||_{L^\infty} < 1$:
        \begin{align*}
        \big|\mathbb{E}[f(b(X))|T=1] - \mathbb{E}[f(b(X))|T=0] \big|
        &= \big|\mathbb{E}[(f \circ b)(X)|T=1] - \mathbb{E}[(f \circ b)(X)|T=0] \big| \\
        & \leq TV(P(X|T=1), P(X|T=0)) 
    \end{align*}
    as $(f \circ b)$ is a function on the $\mathcal{X}$ space (i.e. the image space of $X$) with $||f||_{L^\infty} < 1$. 
        Thus, taking the supremum over all such functions $f$,
    \begin{align*}
    TV(P(b(X)|T=1),P(b(X)|T=0)) \leq TV(P(X|T=1),P(X|T=0)).
    \end{align*}

    \item
    Let $g(\B{\beta}) := \mathbb{E}[f(X)|b(X) = \B{\beta}]$. We show that $g$ is a function on the $\mathcal{B}$ space with $||g||_{L^\infty} < 1$: 
    \begin{align*}
        \forall \B{\beta}, |g(\B{\beta})|
        &= \Big|\mathbb{E}[f(X) \mid b(X) = \B{\beta}]\Big| \\
        &\leq \mathbb{E}[|f(X)| \ \mid \ b(X) = \B{\beta}] \text{ from Jensen's inequality} \\
        &\leq \mathbb{E}[1 \mid b(X) = \B{\beta}] \text{ as } ||f||_{L^\infty} < 1 \\ 
        &= 1.
    \end{align*}
    Therefore, as a consequence of Proposition \ref{th:ipm},
    \begin{align*}
    TV(P(X|T=1),P(X|T=0)) \leq TV(P(b(X)|T=1),P(b(X)|T=0)). 
    \end{align*}
    
    \item Consequently, it follows that 
    \begin{align*}
        TV(P(X|T=1),P(X|T=0)) = TV(P(b(X)|T=1),P(b(X)|T=0)).
    \end{align*}
\end{itemize} $\square$

\begin{repcorollary}{th:balance_X}
Under the same conditions as Proposition~\ref{th:tv_eq_bX}, 
\begin{align*}
&P(b(X) \mid T=1) = P(b(X) \mid T=0) \\
&\implies P(X \mid T=1) = P(X \mid T=0).
\end{align*}
\end{repcorollary}

\textit{Proof:} One should note that $P(b(X) \mid T=1) = P(b(X) \mid T=0)$ implies $TV(P(b(X)|T=1),P(b(X)|T=0)) = 0$. As a consequence, $TV(P(X|T=1),P(X|T=0)) = 0$ from Proposition~\ref{th:tv_eq_bX}. As the total variation is a distance, we obtain that $P(X|T=1) = P(X|T=0)$. $\hspace{1mm} \square$ 

\begin{reptheorem}{th:phew}
Let $b$ be a function such that $b(X)$ is a balancing score, $P'$ be a distribution obtained from matching every treated unit with control units using $b(X)$ only. Then $b(X)$ is also a balancing score in $P'$.
\end{reptheorem}
\textit{Proof:} Let $\B{\beta}$ be a value of $b(X)$. Any matching method using $b(X)$ only to match units does not change the conditional distribution of $X$ given $b(X) = \B{\beta}$ in the control group \citepAppendix{Rosenbaum1983ps}. Then,
\begin{equation} \label{b_control}
    P'(X | b(X) = \B{\beta}, T = 0) = P(X | b(X) = \B{\beta}, T = 0).
\end{equation}
The conditional distribution of $X$ given $b(X) = \B{\beta}$ in the treated group is likewise left unchanged as the matching method does not change the treated distribution in any way. Thus,
\begin{equation} \label{b_treated}
    P'(X | b(X) = \B{\beta}, T = 1) = P(X | b(X) = \B{\beta}, T = 1).
\end{equation} 
Also, as $b(X)$ is a balancing score in $P$,  $P(X | b(X) = \beta, T = 0) =  P(X | b(X) = \beta, T = 1)$. Tying it all together, we have
\begin{align*}
    P'(X | b(X) = \beta, T = 0)
    &= P(X | b(X) = \beta, T = 0) \ \ \ \text{from Equation \ref{b_control}} \\
    &= P(X | b(X) = \beta, T = 1) \ \ \ \text{as }b(X)\text{ is a balancing score.} \\
    &= P'(X | b(X) = \beta, T = 1) \ \ \ \text{from Equation \ref{b_treated}}  \\
\end{align*}
so $b(X)$ is a balancing score in $P'$. $\square$

\subsection{Further Balancing Scores}
\label{app:nsm}

\begin{reptheorem}{th:all_balancing}
Assume that $e(X) = f_L \circ f_{L-1} \circ \dots \circ f_1(X)$ for some functions $f_1, \dots, f_L$. Define $b_0(X) := X$ and $b_l(X) = f_l \circ f_{l-1} \circ \dots \circ f_1(X)$ for $l = 1,\dots,L$. Then, every $b_l(X)$ is a balancing score, and for any $l < L$, $b_{l+1}(X)$ is coarser than $b_l(X)$. 
\end{reptheorem}
\textit{Proof:} According to Proposition \ref{th:ps_balancing}, $b_l(X)$ where $l < L$ is a balancing score as
\begin{align*}
e(X) = f_L \circ f_{L-1} \circ ... \circ f_{l+1}(b_l(X)),
\end{align*}
and $e(X)$ is the propensity score with the property $X \indep T | e(X)$. 
Also for any $l < L$, $b_{l+1}(X)$ is coarser than $b_l(X)$ as $b_{l+1}(X) = f_{l+1}(b_l(X))$.

\begin{repcorollary}{cor:all_balancing_matching}
Under the same conditions and notation as Proposition~\ref{th:all_balancing}, for any $l, l' = 0, \ldots, L$,
\begin{align*}
&TV\left(P(b_{l}(X) \mid T=1),P(b_{l}(X) \mid T=0)\right) \\
&= TV\left(P(b_{l'}(X) \mid T=1),P(b_{l'}(X) \mid T=0)\right),
 \end{align*}
 and balance in $b_{l'}(X)$ is equivalent to balance in $b_l(X)$.
\end{repcorollary}
\textit{Proof:} First, for any $l < L$, as $b_{l+1}(X)$ is a balancing score w.r.t. $b_{l}(X)$ from Proposition \ref{th:all_balancing}, we note that Proposition \ref{th:tv_eq_bX} can also be applied to $b_l(X)$ and $b_{l+1}(X)$ instead of $X$ and $b(X)$, respectively. Thus, it follows from Proposition~\ref{th:tv_eq_bX} that
\begin{align*}
     TV(P(b_{l+1}(X)|T=1),P(b_{l+1}(X)|T=0)) = TV(P(b_{l}(X)|T=1),P(b_{l}(X)|T=0)).
\end{align*}
Consequently, it follows by induction that for any $l, l' = 0, \ldots, L$,
    \begin{align*}
        TV(P(b_l(X)|T=1),P(b_l(X)|T=0)) = TV(P(b_{l'}(X)|T=1),P(b_{l'}(X)|T=0)).
    \end{align*}
    Then, the proof that balance in $b_{l'}(X)$ is equivalent to balance in $b_l(X)$ is analogous to Corollary \ref{th:balance_X}. $\square$

\subsection{Other Integral Probability Metrics}
\label{app:wass_and_mmd}

\begin{reptheorem}{th:ipm}
Let $\mathcal{F}$ be a set of functions of $X$. For any function $b$ of $X$, define
\begin{align*}
\mathcal{F}_b := \left\{ \ \B{\beta} \mapsto \mathbb{E}[f(X) \mid b(X) = \B{\beta}], \ \ f \in \mathcal{F} \right\}.
\end{align*}
Then, for any balancing score $b(X)$ and any set  $\mathcal{G}$ of functions on the image set of $b$ such that $\mathcal{F}_b \subseteq \mathcal{G}$,
\begin{align*}
    &\text{IPM}_{\mathcal{F}}\left(P(X \mid T=1), P(X \mid T=0)\right)\\
    &\leq \text{IPM}_{\mathcal{G}}\left(P(b(X) \mid T=1), P(b(X) \mid T=0)\right) 
\end{align*}
with equality when $\mathcal{G} = \mathcal{F}_b$.
\end{reptheorem}

\begin{proof}
 As $b(X)$ is a balancing score, we have $T \perp \!\!\! \perp X | b(X)$ and for any measurable function $f$:
    \begin{align} 
        \mathbb{E}[f(X) \mid b(X),T] = \mathbb{E}[f(X) \mid b(X)] \label{eq:indep}
    \end{align}
Also, by definition, for any random variable $V$,
    \begin{align}
        \text{IPM}_{\mathcal{F}}(P(V|T=1), P(V|T=0)) = \sup_{f \in \mathcal{F}} \big|\mathbb{E}[f(V) \mid T=1] - \mathbb{E}[f(V) \mid T=0] \big|. \label{eq:ipm}
\end{align}
Let $f$ be a measurable function, then 
    \begin{align}
        \mathbb{E}[f(X) \mid T=t]
        &= \mathbb{E}\big[\mathbb{E}[f(X) \mid b(X),T=t] \ \mid \ T=t\big] \text{ due to the law of total expectation} \nonumber \\
        &= \mathbb{E}\big[\mathbb{E}[f(X) \mid b(X)] \ \mid \ T=t\big] \label{eq:tower_prop_balancing_score} \text{ due to Equation \eqref{eq:indep}} \\
        &= \mathbb{E}\big[g(b(X)) \ | \ T=t\big] \nonumber,
    \end{align}
    where $g(\B{\beta}) = \mathbb{E}[f(X) \mid b(X) = \B{\beta}]$. By definition, if $f \in \mathcal{F}$ then $g \in \mathcal{F}_b$ and, by definition of $\mathcal{G}$, $g \in \mathcal{G}$. Thus, for any $f \in \mathcal{F}$,
        \begin{align*}
        \big|\mathbb{E}[f(X) \mid T=1] - \mathbb{E}[f(X) \mid T=0] \big| &=  \big|\mathbb{E}[g(b(X)) \mid T=1] - \mathbb{E}[g(b(X)) \mid T=0] \big| \text{ for some } g \in \mathcal{G} \\
        &\leq \text{IPM}_{\mathcal{G}}(P(b(X) \mid T=1),P(b(X) \mid T=0)).
    \end{align*}
Taking the supremum wrt $\mathcal{F}$ on the LHS gives that
        \begin{align*}
        \text{IPM}_{\mathcal{F}}(P(X|T=1),P(X|T=0))
        &\leq \text{IPM}_{\mathcal{G}}(P(b(X) \mid T=1),P(b(X) \mid T=0)).
    \end{align*}
Further, let $g \in \mathcal{F}_b$. By definition, there exists $f \in \mathcal{F}$, such that $g(\B{\beta}) = \mathbb{E}[f(X) \mid b(X) = \B{\beta}]$. By Equation \ref{eq:tower_prop_balancing_score},
\begin{align*}
    \forall t = 0, 1, \ \ \mathbb{E}[g(b(X)) \mid T=t] = \mathbb{E}[f(X) \mid T=t].
\end{align*}
Thus,
\begin{align*}
    \left|\mathbb{E}[g(b(X)) \mid T=1] - \mathbb{E}[g(b(X)) \mid T=0] \right| &= \left|\mathbb{E}[f(X) \mid T=1] - \mathbb{E}[f(X) \mid T=0] \right| \text{ for some }f \in \mathcal{F} \\
    &\leq \text{IPM}_{\mathcal{F}}(P(X|T=1), P(X|T=0)).
\end{align*}
Taking the supremum wrt $\mathcal{F}_b$ on the LHS gives
\begin{align*}
    \text{IPM}_{\mathcal{F}_b}(P(b(X) \mid T=1), P(b(X) \mid T=0)) \leq \text{IPM}_{\mathcal{F}}(P(X|T=1), P(X|T=0)),
\end{align*}
concluding the proof.
\end{proof}

\begin{reptheorem}{th:wass_and_mmd}
Let $b$ be a function such that $\forall \B{x}, \ b(\B{x}) = W\B{x}$ for some matrix $W$ and $b(X)$ is a balancing score.
Let $||.||$ be the Euclidean norm on any vector space, and $|||.|||$ be a norm\footnote{Examples include the operator norm or the Euclidean norm.} on any matrix space such that $\forall \B{x}, A, ||A\B{x}|| \leq |||A||| \cdot ||\B{x}||$. 
Further, let $A^+$ be the Moore-Penrose pseudo-inverse of $A$, $\wass$ be the Wasserstein distance, $\mmd$ be the linear MMD\footnote{Note that these theoretical results also hold when $b(X)$ has a bias term.}. Let $W^+_\Sigma := \Sigma W^T (W\Sigma W^T)^+$. If $X$ is elliptical with covariance matrix $\Sigma$ then
\begin{align*}
&\frac{1}{|||W|||} \cdot \mmd\left(P(b(X) \mid T=1),P(b(X) \mid T=0)\right) \\
&\leq \mmd\left(P(X \mid T=1),P(X \mid T=0)\right) \\
&\leq |||W^+_\Sigma|||\cdot \mmd\left(P(b(X) \mid T=1),P(b(X) \mid T=0)\right)
\end{align*}
If $X$ is Gaussian with positive-definite covariance matrix $\Sigma$ and $W$ has full row rank then
\begin{align*}
&\frac{1}{|||W|||} \cdot \wass\left(P(b(X) \mid T=1),P(b(X) \mid T=0)\right) \\
&\leq \wass\left(P(X \mid T=1),P(X \mid T=0)\right) \\
&\leq |||W^+_\Sigma||| \cdot \wass\left(P(b(X) \mid T=1),P(b(X) \mid T=0)\right).
\end{align*}
\end{reptheorem}

\begin{proof}
We prove separately the bounds on the $\wass$ distance and on the $\mmd$.
\begin{itemize}
    \item First, note that for any random variable $V$,
    \begin{align*}
        \wass(P(V|T=1), P(V|T=0)) &= \sup_{f \ \text{1-Lipschitz}} \big|\mathbb{E}[f(V) \mid T=1] - \mathbb{E}[f(V) \mid T=0] \big| \\
        &= \text{IPM}_{\{f: \ f \ \text{1-Lipschitz} \}}( P(V|T=1), P(V|T=0) ),
    \end{align*}
    where IPM is defined in Equation \ref{eq:ipm}.
    \item Let $f$ be a $1$-Lipschitz function $f$ on the $\mathcal{B}$ space of $b(X)$, and define $g(x) = \frac{1}{|||W|||} f(Wx)$. The function $g$ is also 1-Lipschitz, since  for any $x,x'$,
    \begin{align*}
    |g(x) - g(x')| &= \frac{1}{|||W|||} |f(Wx) - f(Wx')| \\
    &\leq \frac{1}{|||W|||} ||Wx - Wx'|| \ \ \ \text{by 1-Lipschitzness of } f \\
    &= \frac{1}{|||W|||} ||W(x - x')|| \\
    &\leq  \frac{1}{|||W|||} \cdot |||W||| \cdot ||x - x'|| \\
    &= ||x - x'||.
    \end{align*}
    Thus,
        \begin{align*}
        \big|\mathbb{E}[f(b(X))|T=1] - \mathbb{E}[f(b(X))|T=0] \big|
        &= |||W||| \cdot \big|\mathbb{E}[g(X)|T=1] - \mathbb{E}[g(X)|T=0] \big| \\
        &\leq |||W||| \cdot \wass(P(X|T=1), P(X|T=0)).
    \end{align*}

    It follows that
    \begin{align*}
    \wass(P(b(X)|T=1),P(b(X)|T=0)) \leq |||W||| \cdot \wass(P(X|T=1),P(X|T=0)) .
    \end{align*}
    
    \item Now, let $f$ is a $1$-Lipschitz real-valued function on $\mathcal{X}$. We show that $g(\B{\beta}) = \mathbb{E}[f(X)|b(X) = \B{\beta}]$ is Lipschitz. First, as $\text{Cov}(X, WX) = \Sigma W^T$, $\text{Cov}(WX, X) = W\Sigma$, and $\text{Cov}(WX, WX) = W\Sigma W^T$ which is invertible as $W$ has full row rank, we have that for any $\B{\beta}$, 
    \begin{align*}
    X | WX = \B{\beta} \sim \mathcal{N}\left(\mu + \Sigma W^T (W\Sigma W^T)^{-1} (\B{\beta} - W\mu), \Sigma - \Sigma W^T (W\Sigma W^T)^{-1} W\Sigma \right).
    \end{align*}

    As a result, for any $\B{\beta}$,
    \begin{align*}
        g(\B{\beta}) = \mathbb{E}_{Z \in \mathcal{N}((I - \Sigma W^T (W\Sigma W^T)^{-1} W)\mu, \Sigma - \Sigma W^T (W\Sigma W^T)^{-1} W\Sigma)}[f(\Sigma W^T (W\Sigma W^T)^{-1} \B{\beta} + Z)]
    \end{align*}
    We simplify the notation of this expectation into $\mathbb{E}_Z[...]$. Note that, critically, the distribution of $Z$ does not depend on $\B{\beta}$.

    Then, for any $\B{\beta}, \B{\beta'}$,
    \begin{align*}
        |g(\B{\beta}) - g(\B{\beta'})|
        &= \left|\mathbb{E}_Z[f(\Sigma W^T (W\Sigma W^T)^{-1} \B{\beta} + Z)] - \mathbb{E}_Z[f(\Sigma W^T (W\Sigma W^T)^{-1} \B{\beta'} + Z)]\right| \\
        &= \left|\mathbb{E}_Z[f(\Sigma W^T (W\Sigma W^T)^{-1} \B{\beta} + Z) - f(\Sigma W^T (W\Sigma W^T)^{-1} \B{\beta'} + Z)]\right| \\
        &\leq  \mathbb{E}_Z[\left|f(\Sigma W^T (W\Sigma W^T)^{-1} \B{\beta} + Z) - f(\Sigma W^T (W\Sigma W^T)^{-1} \B{\beta'} + Z)\right|] \text{ from Jensen's inequality} \\
        &\leq \mathbb{E}_Z[||(\Sigma W^T (W\Sigma W^T)^{-1} \B{\beta} + Z) - (\Sigma W^T (W\Sigma W^T)^{-1} \B{\beta'} + Z)||] \ \ \ \text{from the 1-Lipschitzness of $f$} \\
        &=  \mathbb{E}_Z[||\Sigma W^T (W\Sigma W^T)^{-1} \B{\beta} - \Sigma W^T (W\Sigma W^T)^{-1} \B{\beta'}||] \\
        &= ||\Sigma W^T (W\Sigma W^T)^{-1} \B{\beta} - \Sigma W^T (W\Sigma W^T)^{-1} \B{\beta'}|| \\
        &= ||\Sigma W^T (W\Sigma W^T)^{-1}( \B{\beta} - \B{\beta'})|| \\
        &\leq |||\Sigma W^T (W\Sigma W^T)^{-1}||| \cdot ||\B{\beta} - \B{\beta'}|| 
    \end{align*}
    so $g$ is $|||\Sigma W^T (W\Sigma W^T)^{-1}|||$-Lipschitz. Therefore, as a consequence of Proposition \ref{th:ipm},
    \begin{align*}
    \wass(P(X|T=1),P(X|T=0)) \leq |||\Sigma W^T (W\Sigma W^T)^{-1}||| \cdot \wass(P(b(X)|T=1),P(b(X)|T=0)).
    \end{align*}
    \item For MMD, note that for any random variable $V$ :
    \begin{align*}
        \mmd(P(V|T=1), P(V|T=0)) &= \sup_{a \in \mathbb{R}^{\text{dim}(V)}, \ \ ||a|| \leq 1} \big|\mathbb{E}[a^TV|T=1] - \mathbb{E}[a^TV|T=0] \big| \\
        &= ||\mathbb{E}[V|T=1] - \mathbb{E}[V|T=0] ||
    \end{align*}

    \item We note that
    \begin{align*}
    &\mmd\left(P(b(X) \mid T=1),P(b(X) \mid T=0)\right) \\
    &= ||\mathbb{E}[b(X)|T=1] - \mathbb{E}[b(X)|T=0] || \\
    &= ||\mathbb{E}[WX|T=1] - \mathbb{E}[WX|T=0] || \\
    &= ||W(\mathbb{E}[X|T=1] - \mathbb{E}[X|T=0])|| \\
    &\leq |||W||| \cdot ||\mathbb{E}[X|T=1] - \mathbb{E}[X|T=0]|| \\
    &= |||W||| \cdot \mmd\left(P(X \mid T=1),P(X \mid T=0)\right) \\
    \end{align*}
    
    \item From Equation \ref{eq:tower_prop_balancing_score}, $\mathbb{E}[X|T=t] = \mathbb{E}\big[g(b(X)) \ | \ T=t\big]$, where $g(\B{\beta}) := \mathbb{E}[X|b(X) = \B{\beta}]$. From Section 2 of \citepAppendix{cambanis1981ottoecd}, if if $X$ is elliptical with location $\mu$ and covariance matrix $\Sigma$ then $\begin{pmatrix} I \\ W \end{pmatrix}X$ is elliptical with location $\begin{pmatrix} I \\ W \end{pmatrix}\mu$ and covariance matrix $\begin{pmatrix} \Sigma & \Sigma W^T \\ W\Sigma & W\Sigma W^T \end{pmatrix}$. Then, from Corollary 5 of \citepAppendix{cambanis1981ottoecd},
    \begin{align*}
    g(\B{\beta}) &= \mathbb{E}[X|WX = \B{\beta}] \\
    &= \mu + \Sigma W^T (W\Sigma W^T)^+(\B{\beta} - W\mu).
    \end{align*}
    Thus,
    \begin{align*}
        \mathbb{E}[X|T=t]
        &= \mathbb{E}\big[g(b(X)) \ | \ T=t\big] \\
        &= \mathbb{E}\big[\mu + \Sigma W^T (W\Sigma W^T)^+(b(X) - W\mu) \ | \ T=t\big] \\
        &= \Sigma W^T (W\Sigma W^T)^+\mathbb{E}[b(X) \ | \ T=t] + \left(I  - \Sigma W^T (W\Sigma W^T)^+W\right)\mu
    \end{align*}
    Then,
    \begin{align*}
    &\mmd\left(P(X \mid T=1),P(X \mid T=0)\right) \\
    &= || \mathbb{E}[X|T=1] - \mathbb{E}[X|T=0] || \\
    &= \Big|\Big|\Sigma W^T (W\Sigma W^T)^+\mathbb{E}[b(X) \ | \ T=1] + \left(I  - \Sigma W^T (W\Sigma W^T)^+W\right)\mu \\
    & \ \ \ \  - \left(\Sigma W^T (W\Sigma W^T)^+\mathbb{E}[b(X) \ | \ T=0] + \left(I  - \Sigma W^T (W\Sigma W^T)^+W\right)\mu\right) \Big|\Big| \\
    &= ||\Sigma W^T (W\Sigma W^T)^+(\mathbb{E}[b(X) \ | \ T=1] - \mathbb{E}[b(X) \ | \ T=0]) || \\
    &\leq |||\Sigma W^T (W\Sigma W^T)^+||| \cdot ||\mathbb{E}[b(X)|T=1] - \mathbb{E}[b(X)|T=0]|| \\
    &= |||\Sigma W^T (W\Sigma W^T)^+||| \cdot \mmd\left(P(b(X) \mid T=1),P(b(X) \mid T=0)\right). 
    \end{align*} 
\end{itemize} 
\end{proof}

\subsection{Bounds For Non-Balancing Scores}
\label{app:approx}

\begin{reptheorem}{th:approx}
Let
\begin{align*}
\mathcal{E}^D_{t,b}(\B{\beta}) := D\Big( P\big(X | b(X) = \B{\beta} , T=t\big), P\big(X |  b(X) = \B{\beta} \big) \Big)
\end{align*}
where $D$ is a probability discrepancy measure, $b$ is a function of $X$, $t \in \{0,1\}$ is a realisation of $T$, $\B{\beta}$ is a realisation of $b(X)$. For any function $b$,
     \begin{align*}
&TV\Big( P\big(b(X) | T=1\big), P\big(b(X) | T=0\big) \Big) \\
&\leq TV\Big( P\big(X | T=1\big), P\big(X | T=0\big) \Big) \\
& \leq TV\Big( P\big(b(X) | T=1\big), P\big(b(X) | T=0\big) \Big) \\
& + \mathbb{E}\big[\mathcal{E}^{TV}_{1,b}\big(b(X)\big) | T=1\big] + \mathbb{E}\big[\mathcal{E}^{TV}_{0,b}\big(b(X)\big) | T=0\big]
\end{align*}
and, using the notations of Proposition \ref{th:ipm},
\begin{align*}
    &\text{IPM}_{\mathcal{F}}\left(P(X \mid T=1), P(X \mid T=0)\right)\\
    &\leq \text{IPM}_{\mathcal{F}_b}\left(P(b(X) \mid T=1), P(b(X) \mid T=0)\right) \\
    &+ \mathbb{E}\big[\mathcal{E}^{\text{IPM}_{\mathcal{F}}}_{1,b}\big(b(X)\big) | T=1\big] + \mathbb{E}\big[\mathcal{E}^{\text{IPM}_{\mathcal{F}}}_{0,b}\big(b(X)\big) | T=0\big].
\end{align*}

For a linear function  $b(x) = Wx$, if $X$ is elliptical with covariance matrix $\Sigma$, then
\begin{align*}
&\frac{1}{|||W|||} \cdot
 \mmd\Big( P\big(b(X) | T=1\big), P\big(b(X) | T=0\big) \Big) \\
&\leq \mmd\Big( P\big(X | T=1\big), P\big(X | T=0\big) \Big) \\
& \leq |||W^+_\Sigma||| \cdot
 \mmd\Big( P\big(b(X) | T=1\big), P\big(b(X) | T=0\big) \Big) \\
 &
 + \mathbb{E}\big[\mathcal{E}^{\mmd}_{1,b}\big(b(X)\big) | T=1\big] + \mathbb{E}\big[\mathcal{E}^{\mmd}_{0,b}\big(b(X)\big) | T=0\big]
\end{align*}
and if $X$ is Gaussian with positive-definite covariance matrix $\Sigma$ while $W$ has full row rank, then
\begin{align*}
&\frac{1}{|||W|||} \cdot
 \wass\Big( P\big(b(X) | T=1\big), P\big(b(X) | T=0\big) \Big) \\
& \leq \wass\Big( P\big(X | T=1\big), P\big(X | T=0\big) \Big) \\
& \leq |||W^+_\Sigma||| \cdot
 \wass\Big( P\big(b(X) | T=1\big), P\big(b(X) | T=0\big) \Big) \\
 &
 + \mathbb{E}\big[\mathcal{E}^{\wass}_{1,b}\big(b(X)\big) | T=1\big] + \mathbb{E}\big[\mathcal{E}^{\wass}_{0,b}\big(b(X)\big) | T=0\big].
\end{align*}
\end{reptheorem}

\begin{proof}
The lower bounds were established in the previous Propositions, while the upper bounds follow as a corollary of the following Proposition. 
Indeed, the proofs Propositions \ref{th:tv_eq_bX}, \ref{th:ipm} and \ref{th:wass_and_mmd} directly show that Equation \eqref{eq:general_condition_approx} follows for their respective assumptions on classes of functions, distributions of $X$ and balancing scores. 
\end{proof}

\begin{theorem}  \label{th:intermediate_prop_1}
Let $b$ a function of $X$, $\mathcal{F}$ a class of functions of $X$. Assume that for some constant $C_b$ and some class of function $\mathcal{F}'_b$ of functions on the image space on $b$, both depending on $b$:
\begin{align} \label{eq:general_condition_approx}
    \forall f \in \mathcal{F}, \Big| \mathbb{E}\big[ \ \mathbb{E}[f(X) | b(X)]  \ \big| \ T = 1] - \mathbb{E} \big[ \ \mathbb{E}[f(X) | b(X)] \ \big| \ T = 0]  \Big| \leq C_b \cdot \text{IPM}_{\mathcal{F}'_b}\Big( P\big(b(X) | T=1\big), P\big(b(X) | T=0\big) \Big).
\end{align}
Then, letting $\mathcal{E}^D_{t,b}(\B{\beta}) = D\Big( P\big(X | b(X) = \B{\beta} , T=t\big), P\big(X |  b(X) = \B{\beta} \big) \Big)$ where $D$ is a probability distance, we have
\begin{align*}
\text{IPM}_\mathcal{F}\Big( P\big(X | T=1\big), P\big(X | T=0\big) \Big) \leq \ 
& C_b \cdot \text{IPM}_{\mathcal{F}'_b}\Big( P\big(b(X) | T=1\big), P\big(b(X) | T=0\big) \Big) \\
& + \mathbb{E}\big[\mathcal{E}^{\text{IPM}_\mathcal{F}}_{1,b}\big(b(X)\big) | T=1\big] + \mathbb{E}\big[\mathcal{E}^{\text{IPM}_\mathcal{F}}_{0,b}\big(b(X)\big) | T=0\big]
\end{align*}
\end{theorem}

\begin{proof} 
Denote $\Delta_t(\B{\beta}; f) := \mathbb{E}\Big[ f(X) \Big| b(X) = \B{\beta}, T = t \Big] - \mathbb{E}\Big[ f(X) \Big| b(X) = \B{\beta} \Big]$, so that $\mathcal{E}^{\text{IPM}_\mathcal{F}}_{t,b}\big(\B{\beta}\big) = \sup_{f \in \mathcal{F}} \Big| \Delta_t(\B{\beta}; f) \Big|$.

We fix $f \in \mathcal{F}$, noting that
\begin{align*}
        \mathbb{E}[f(X)|T=t]
        &= \mathbb{E}\big[\mathbb{E}[f(X)|b(X),T=t] \ \mid \ T=t\big] \text{ due to the law of total expectation} \\
        &= \mathbb{E}\big[\Delta_t\big(b(X); f\big) + \mathbb{E}[f(X)|b(X)] \ \mid \ T=t\big] \\
        &= \mathbb{E}\big[\Delta_t\big(b(X); f\big) \ \mid \ T=t\big] \ + \ \mathbb{E}\big[\mathbb{E}[f(X)|b(X)] \ \mid \ T=t\big].
\end{align*}
As a consequence,
    \begin{align*}
        &\Big|\mathbb{E}[f(X)|T=1] - \mathbb{E}[f(X)|T=0]\Big| \\
        &= \Big| \ \mathbb{E}\big[\mathbb{E}[f(X)|b(X)]  \ \mid \ T=1\big] - \mathbb{E}\big[\mathbb{E}[f(X)|b(X)]  \ \mid \ T=0\big] + \mathbb{E}\big[\Delta_1\big(b(X); f\big) \ \mid \ T=1\big] \\
        & \ \ \ \ \ - \mathbb{E}\big[\Delta_0\big(b(X); f\big) \ \mid \ T=0\big] \ \Big| \\
        &\leq \Big| \ \mathbb{E}\big[\mathbb{E}[f(X)|b(X)]  \ \mid \ T=1\big] - \mathbb{E}\big[\mathbb{E}[f(X)|b(X)]  \ \mid \ T=0\big] \ \Big| \ + \ \Big| \ \mathbb{E}\big[\Delta_1\big(b(X); f\big) \ \mid \ T=1\big] \ \Big| \ \\
        & \ \ \ \ \ + \ \Big| \ \mathbb{E}\big[\Delta_0\big(b(X); f\big) \ \mid \ T=0\big] \ \Big|,
\end{align*}
where
\begin{align*}
\Big| \ \mathbb{E}\big[\mathbb{E}[f(X)|b(X)]  \ \mid \ T=1\big] - \mathbb{E}\big[\mathbb{E}[f(X)|b(X)]  \ \mid \ T=0\big] \ \Big| &\leq C_b \cdot \text{IPM}_{\mathcal{F}'_b}\Big( P\big(b(X) | T=1\big), P\big(b(X) | T=0\big) \Big) \\
& \ \ \ \ \ \ \ \ \ \text{by assumption}
\end{align*}
and, for $t \in \lbrace 0, 1 \rbrace$,
\begin{align*}
\Big| \ \mathbb{E}\big[\Delta_t\big(b(X); f\big) \ \mid \ T=t\big] \ \Big| \ &\leq  \ \mathbb{E}\big[ \ | \Delta_t\big(b(X); f\big) | \ \big| \ T=t\big] \\ &\leq \ \mathbb{E}\big[ \sup_{f \in \mathcal{F}} | \Delta_t\big(b(X); f\big) | \ \big| \ T=t\big] \\ &= \ \mathbb{E}\big[\mathcal{E}^{\text{IPM}_\mathcal{F}}_{t,b}\big(b(X)\big) | T=t \big].
\end{align*}
Thereby, for any $f \in \mathcal{F}$,
\begin{align*}
\Big|\mathbb{E}[f(X)|T=1] - \mathbb{E}[f(X)|T=0]\Big| \leq \ 
& C_b \cdot \text{IPM}_{\mathcal{F}'_b}\Big( P\big(b(X) | T=1\big), P\big(b(X) | T=0\big) \Big) \\
& + \mathbb{E}\big[\mathcal{E}^{\text{IPM}_\mathcal{F}}_{1,b}\big(b(X)\big) | T=1\big] + \mathbb{E}\big[\mathcal{E}^{\text{IPM}_\mathcal{F}}_{0,b}\big(b(X)\big) | T=0\big].
\end{align*}
Taking the supremum over $f \in \mathcal{F}$ yields the desired result.

\end{proof}

\section{A FEW NOTES ABOUT COMPUTATIONAL COMPLEXITIES OF BOUNDS}
\label{app:complexity_wass}

\paragraph{Computational complexity of bounds in Proposition \ref{th:wass_and_mmd}}{Denoting $N := N_t + N_c$, the computational complexity of the linear MMD estimator is in $\mathcal{O}(ND)$, and the computational complexity of the the Wasserstein distance estimator is in $\mathcal{O}(N^2D + N^3 \log N + N^3 \log D )$ when using the auction algorithm \citepMain{peyre2020computational, Bertsekas98networkoptimization}, assuming that covariates and balancing scores have bounded second-order moments; we refer to the paragraph below on ``Computational complexity of Wasserstein distance''. As a result, assuming that the balancing score is of dimension $d << D$ and that we have already computed the ground-truth balancing scores, these complexities can be decreased to $\mathcal{O}(Nd)$ and $\mathcal{O}(N^2d + N^3 \log N + N^3 \log d)$, respectively.  Thus, if we assume $d << D \sim N$, there is a clear decrease of the complexity for the linear MMD. The decrease is less stark for the Wasserstein distance, as the dominant $N^3 \log N$ term is untouched; however other terms are clearly decreased. 

The decrease of complexity should be nuanced if we compute the entire bounds of Proposition \ref{th:wass_and_mmd}, and not just the probability distances, as we have to (1) compute the constants $|||W|||$ and $|||W^+_\Sigma|||$, and (2) compute the balancing scores $(WX_i)_i$. We assume that $|||.|||$ is the operator norm. For (1), when $\Sigma = I$, then $W^+_\Sigma = W^+$, both constants can be handled simultaneously by computing the singular value decomposition of $W$, which has a complexity $\mathcal{O}(Dd^2)$ \citepAppendix{golub2013matrix, Vasudevan2017AHS}. For a general $\Sigma$, then $W^+_\Sigma = \Sigma W^T(W\Sigma W^T)^{+}$ and computing $|||W^+_\Sigma|||$ has a complexity $\mathcal{O}(D^2d + Dd^2 + d^3)$, as computing $\Sigma W^T$ (present twice in the formula of $W^+_\Sigma$) is in $\mathcal{O}(D^2d)$, further computing $W\Sigma W^T$ is in $\mathcal{O}(d^2D)$, computing the pseudo-inverse through computing the singular value decomposition is in $\mathcal{O}(d^3)$, deducing $W^+_\Sigma$ from both the inverse matrix and the already computed $\Sigma W^T$ is in $\mathcal{O}(Dd^2)$, and another singular value decomposition for the norm is in $\mathcal{O}(Dd^2)$. For (2), we further increase computational complexity by a term $\mathcal{O}(NDd)$ due to the additional matrix multiplication operations. As a result, when $d << D \sim N$, the bounds for the linear MMD imbalance actually have higher computational complexity than the original imbalance itself, while those for the Wasserstein distance imbalance have slightly lower computational complexity than the original imbalance itself.}

\paragraph{Computational complexity of Wasserstein distance}{More precisely, the computational complexity of the Wasserstein distance is $\mathcal{O}(N^2 D + \min\{N^3\log C_{\infty,X}, N^2 C_{\infty,X}^2 \log N \} )$, where the first term corresponds to computing the $L_2$ distance matrix wrt $X$, and the second term corresponds to the minimum of the computational complexities of the auction algorithm \citepAppendix{peyre2020computational, Bertsekas98networkoptimization} and Sinkhorn's algorithm \citepAppendix{dvurechensky2018computational}, assuming we choose the algorithm with the lowest complexity. $C_{\infty,X}$ is an upper bound of the maximal value of the distance matrix wrt $X$ and can further depend on $N$ and $D$. 
We assume covariates have a bounded second-order moment: noting $X_i$ covariates of treated units, $X'_j$ those of control units, $k$ the dimension index, we assume that $\forall i,k, \ \ \mathbb{E}[|X_i^k|^2] < M$ and $\forall j,k, \ \ \mathbb{E}[|X_j^{\prime k}|^2] < M$. Then 
\begin{align*}
\mathbb{E}[C_{\infty,X}]
    &= \mathbb{E}[\max_{i,j} ||X_i - X'_j||] \\
    &=  \mathbb{E}\left[\max_{i,j} \sqrt{\sum_{k=1}^D |X_i^k - X_j^{\prime k}|^2}\right] \\
    &=  \mathbb{E}\left[\sqrt{\max_{i,j} \sum_{k=1}^D |X_i^k - X_j^{\prime k}|^2}\right] \\
    &\leq \sqrt{\mathbb{E}\left[\max_{i,j} \sum_{k=1}^D |X_i^k - X_j^{\prime k}|^2\right] } \text{ from Jensen's inequality as $\sqrt{.}$ is concave} \\
    &\leq  \sqrt{\mathbb{E}\left[\sum_{k=1}^D \max_{i,j}  |X_i^k - X_j^{\prime k}|^2\right]} \\
    &\leq  \sqrt{\mathbb{E}\left[\sum_{k=1}^D \max_{i,j}  2(|X_i^k|^2 + |X_j^{\prime k}|^2)\right]}  \ \ \text{ from } (a-b)^2 \leq 2(a^2+b^2) \ \forall a,b \\
    &=  \sqrt{\mathbb{E}\left[\sum_{k=1}^D 2( \max_i |X_i^k|^2 + \max_j |X_j^{\prime k}|^2 )\right]} \\
    &= \sqrt{\sum_{k=1}^D 2\mathbb{E}\left[ \max_i |X_i^k|^2 + \max_j |X_j^{\prime k}|^2 \right]} \\
    &\leq  \sqrt{ 2\sum_{k=1}^D \mathbb{E}\left[ \sum_i |X_i^k|^2 + \sum_j |X_j^{\prime k}|^2 \right] }\\
    &=  \sqrt{  2\sum_{k=1}^D \big( \sum_i \mathbb{E}\left[|X_i^k|^2\right] + \sum_j \mathbb{E}\left[|X_j^{\prime k}|^2\right] \big) } \\
    &\leq \sqrt{2 \cdot D \cdot (N_t + N_c) \cdot M} \\
    &= \sqrt{2 \cdot D \cdot N \cdot M} \\
\end{align*}
so, from Jensen's inequality applied to the log function,
\begin{align*}
    \mathbb{E}[\log C_{\infty,X}] \leq \log \mathbb{E}[C_{\infty,X}] = \frac{1}{2} \cdot (\log 2 + \log N + \log D + \log M)
\end{align*}
and
\begin{align*}
    \mathbb{E}[C_{\infty,X}^2]
    &= \mathbb{E}\left[\left(\max_{i,j} ||X_i - X'_j||\right)^2\right] \\
    &=  \mathbb{E}\left[\left(\max_{i,j} \sqrt{\sum_{k=1}^D |X_i^k - X_j^{\prime k}|^2}\right)^2\right] \\
        &=  \mathbb{E}\left[\max_{i,j}\left( \sqrt{\sum_{k=1}^D |X_i^k - X_j^{\prime k}|^2}\right)^2\right] \\
&=  \mathbb{E}\left[\max_{i,j} \sum_{k=1}^D |X_i^k - X_j^{\prime k}|^2\right] \\
&\leq 2DNM. \\
\end{align*}
where we repeated the above expectations from after Jensen's inequality without the square root. 
 Thus, assuming $D \sim N$ or $D \leq N$ and substituting those complexities in expectation into the computational complexities above, the auction algorithm is in $\mathcal{O}(N^3 \log N + N^3 \log D )$ in expectation, and Sinkhorn's algorithm is in $\mathcal{O}(N^3D\log N)$ in expectation, so the auction algorithm might be preferable. \\

\section{IMPLEMENTATION DETAILS}
\label{app:Implementation Details}

\paragraph{ACIC 2016 Dataset. }{This dataset is taken from the ACIC competition of 2016 \citepAppendix{dorie2017acic}. 
Covariates were obtained from a study about developmental disorders, measuring data from pregnant women and their children. 
Treatment assignments and outcomes were synthetically generated from transformed versions of covariates using different data generating processes. 
Importantly, as treatments are synthetically generated, ground-truth propensity scores are made readily available, allowing us to compute calibration errors. 
We chose the provided data generating process setting number 4, which has polynomial treatment assignment, an exponential outcome model, $35\%$ of treated units, full overlap, and high treatment heterogeneity. 
To preprocess the data, categorical covariates with $F$ factors were converted to $F-1$ binary covariates, where the $f$-th binary covariate encodes factor $f+1$. 
Due to high heterogeneity between subjects, we also centered and scaled continuous covariates to improve performance of all models.
Binary covariates were left unprocessed. 
4802 subjects were present in the dataset. 
The subjects have 82 covariates after preprocessing (23 continuous and 59 binary). 
In our experiments, we considered 100 versions of this dataset, each corresponding to a different random seed for the data generating process.}

\paragraph{News Dataset. }{This dataset contains 5000 documents extracted from the NYT Corpus, where each of the $3477$ covariates represents counts of a word in news articles. 
The treatment indicator $T$ represents the use of a desktop ($T = 0$) or a mobile device ($T = 1$). 
The real-valued outcome $Y$ measures the opinion of the reader about the news article. 
Both treatments and outcomes are generated using a data generating process. 
Here, 50 random seeds from the data generating process are considered. 
In contrast to ACIC 2016, we did not choose these random seeds ourselves as they were already provided by the original authors\footnote{See ``News'' link in the ``Software and Data'' section here: https://www.fredjo.com/}~\citepAppendix{johansson2016learning}.}

\paragraph{IHDP Dataset.}{For this dataset, covariates and treatment assignments are used from 747 subjects in real-world data of the Infant Health Development Program. 
Outcomes, however, are synthetically generated. We further apply the same scaling of outcomes as in \citeAppendix{Curth2021nonparametric}, as the absence of scaling led to a few outliers causing very high ATT errors in all methods, making comparisons very challenging.
Here, 50 seeds from the data generating process are considered, directly used from the implementation of Dragonnet \citepAppendix{Shi2019adapting}. 25 covariates are present (9 are continuous, 16 are binary).
Experimental results on this dataset are presented in Appendix \ref{app:ihdp}.
}

\paragraph{Evaluation Metrics. }{To evaluate and compare experimental results, we use the following metrics: 
\begin{itemize}[noitemsep,topsep=0pt,parsep=0pt,partopsep=0pt,leftmargin=*]
    \item The \textit{calibration error}, defined as the mean absolute difference between the estimated and true propensity score. 
    This metric can only be computed when the true propensity score is assumed to be known in the dataset.
    The smaller the calibration error, the more suitable the estimated propensity score and estimated balancing scores obtained from a model are for matching, as we will be closer to the assumption that the propensity score is correctly estimated. Connecting the calibration error to the balancing error term in Proposition \ref{th:approx} is left for future work.
    \item The \textit{ATT error}, defined as the absolute difference between the ATT estimated by the method and a ground-truth ATT. 
    For every dataset, we compute the ground-truth ATT as the approximation from Equation~\eqref{eq:att_gt}, as we have access to the conditional expectations of $Y$.
    \item We empirically quantify \textit{sample imbalance} $\hat{I}$, defined as the squared Euclidean distance between sample means of covariates of treated and control groups from the %
    dataset $\mathcal{D}'$, which is obtained from the original dataset $\mathcal{D}$ via matching, or formally,
    \begin{align*}
        \hat{I} = \Big|\Big| \frac{1}{N_t}\sum_{i \in \mathcal{D} : T_i = 1} X_i - \frac{\sum_{j \in \mathcal{D} : T_j = 0} w_jX_j}{\sum_{j \in \mathcal{D} : T_j = 0} w_j}  \Big|\Big|_2^2, \nonumber
    \end{align*}
    where $N_t$ is the number of treated samples, and $w_j$ is the total weight of control sample $j$ after matching.
    As we can see from this equation, only the sample means of covariates from the control group may change due to matching%
    ; the sample means of covariates from the treated group remain unchanged.
    We note that this measure of imbalance is proportional to the squared linear MMD \citepAppendix{Sriperumbudur2012ipms}.
\end{itemize}
}

\paragraph{Data Splits. }{The neural networks were trained using a 60/20/20 training/validation/testing split. 
The benchmarks logistic regression-based propensity score estimate and PCA were trained using the combined training and validation sets. 
In-sample metrics were also computed on the combined training and validation datasets, and hold-out metrics were evaluated using the testing set. 
Alternatively, one might also use controls from the in-sample set when computing hold-out metrics. 
However, for simplicity of the definition of the hold-out imbalance, we preferred to just use controls from the testing set.}

\paragraph{Neural Architecture. }{The architecture of the neural networks used for matching is as follows : a low-dimensional layer corresponding to the multivariate balancing score (which we also call the "balancing score layer"), then wide hidden layers which are not used as balancing scores, and finally the propensity score head. This architecture is designed to focus on a linear balancing score as in Proposition \ref{th:wass_and_mmd} while keeping flexibility in the rest of the architecture to fit the propensity score model.}

\paragraph{Hyperparameters. }{To choose hyperparameters, we ran a grid search over the following hyperparameter values, minimising validation error on the first dataset version of ACIC 2016 (setting 4, as discussed above).}

\begin{itemize}
    \item Number of hidden layers in addition to the balancing score (hidden) layer: 1, 2.
 \item Number of hidden units per hidden layer (besides the balancing score layer): 100, 200, 300.
 \item Learning rate: $10^{-2}$, $10^{-3}$, $10^{-4}$.
 \item Weight decay: 0, 0.001, 0.01.
\end{itemize}
Other hyperparameters which we did not tune include a batch size of 100, and stochastic gradient descent with fixed learning rate as the optimiser. 
The chosen values by the hyperparameter search were 2 hidden layers besides the balancing score layer, 100 hidden units per hidden layer other than the balancing score layer, a learning rate of $10^{-2}$, weight decay with 0.01, and leaky ReLU as an activation. 
Additionally, on News datasets, the chosen hyperparameters caused the validation loss to diverge after a period of decrease, causing the training to fail. 
Thus, for this dataset, we used early stopping as a remedy.

\paragraph{Code. }{We provide our code to implement neural score matching and reproduce our main results at \textcolor{blue}{\href{https://github.com/oscarclivio/neuralscorematching}{\url{https://github.com/oscarclivio/neuralscorematching}}}.}

\paragraph{Resources and Assets. }{Experiments were run on a laptop with a GeForce GTX 1070 GPU with Max-Q Design for training models with neural networks, and on 12 CPU cores for other tasks. 
For all datasets, we used our own implementation of them in NumPy and PyTorch (after downloading the data in the case of ACIC 2016 and IHDP, as discussed above), and used our own PyTorch implementation for neural network training.}

\section{IHDP}
\label{app:ihdp}

In addition to the experimental results in the main paper, we also provide results for the IHDP dataset \citepAppendix{hill2011ihdp} in Table \ref{ihdp}. 
Boxplots are presented in Section \ref{app:boxplots}.

On IHDP, our method is not outperforming other methods. 
Plain covariates $\texttt{X}$ consistently rank as the best or second best method for each metric and setting (in-sample or hold-out). 
This might indicate that IHDP, which is a rather low-dimensional dataset with only 25 covariates, is not suited for dimensionality reduction methods, but further work should investigate these results. 
We also note that matching in the raw covariate space is probably facilitated by the fact that 16 of covariates are binary.

\begin{table}[ht!]
\centering
\caption{\label{ihdp} Results on the IHDP dataset.}
\setlength{\tabcolsep}{2pt}
\begin{tabular}{lcc}
\toprule
ATT errors & In-Sample & Hold-Out \\
\midrule
\texttt{NN Layer 1} (ours) & 0.156$\pm$0.005 & 0.311$\pm$0.011 \\
\texttt{NN PS} (ours) & 0.190$\pm$0.006 & 0.330$\pm$0.011 \\
\texttt{X} & 0.144$\pm$0.005 & 0.295$\pm$0.011 \\
\texttt{Random matching} & 0.216$\pm$0.007 & 0.342$\pm$0.012 \\
\texttt{LogReg PS} & 0.164$\pm$0.005 & 0.294$\pm$0.009 \\
\texttt{PCA} & 0.159$\pm$0.005 & 0.307$\pm$0.011 \\
\texttt{PCA + LogReg PS} & 0.146$\pm$0.005 & 0.372$\pm$0.011 \\
\midrule
Imbalances & In-Sample & Hold-Out \\
\midrule 
\texttt{NN Layer 1} (ours) & 0.159$\pm$0.005 & 0.442$\pm$0.009 \\
\texttt{NN PS} (ours) & 0.335$\pm$0.006 & 0.511$\pm$0.008 \\
\texttt{X} & 0.07$\pm$0.000 & 0.223$\pm$0.000 \\
\texttt{Random matching} & 0.592$\pm$0.006 & 0.658$\pm$0.012 \\
\texttt{LogReg PS} & 0.033$\pm$0.000 & 0.318$\pm$0.000 \\
\texttt{PCA} & 0.129$\pm$0.000 & 0.407$\pm$0.000 \\
\texttt{PCA + LogReg PS} & 0.137$\pm$0.001 & 0.909$\pm$0.003 \\
\texttt{No Matching} & 0.492$\pm$0.000 & 0.421$\pm$0.000 \\
\bottomrule
\end{tabular}
\end{table}

\section{BOXPLOTS OF ATT ERRORS AND IMBALANCES}
\label{app:boxplots}

We show boxplots corresponding to Tables \ref{tab:acic2016} to \ref{ihdp} in Figures~\ref{fig:acic2016_calibration_errors} to \ref{fig:ihdp_imbalances}. 
We provide boxplots with and without outliers. Outliers are defined as values above $Q3 + 1.5 \cdot IQ$ and below $Q1 - 1.5 \cdot IQ$ where $Q1, Q3, IQ$ are the lower quartile, the upper quartile and the interquartile range of the underlying data, respectively.

\section{SOCIETAL IMPACT}
\label{app:impact}

Possible positive societal impacts of our method include improving decision-making for various real-world applications in politics, economics or medicine. 
Possible negative societal impacts include the misuse of individualised treatment effect estimation to discriminate against individuals or groups, and of matching to identify protected characteristics of individuals or groups. 
To mitigate such impacts, we emphasise the importance of continued oversight and evaluation in the deployment of AI tools in society as well as the protection of data confidentiality via rigorous anonymisation, particularly with regards to protected characteristics.

\bibliographystyleAppendix{apalike}%
\bibliographyAppendix{references.bib}

\newpage

\begin{figure}[ht!]
    \centering
    \includegraphics[width=0.8\linewidth]{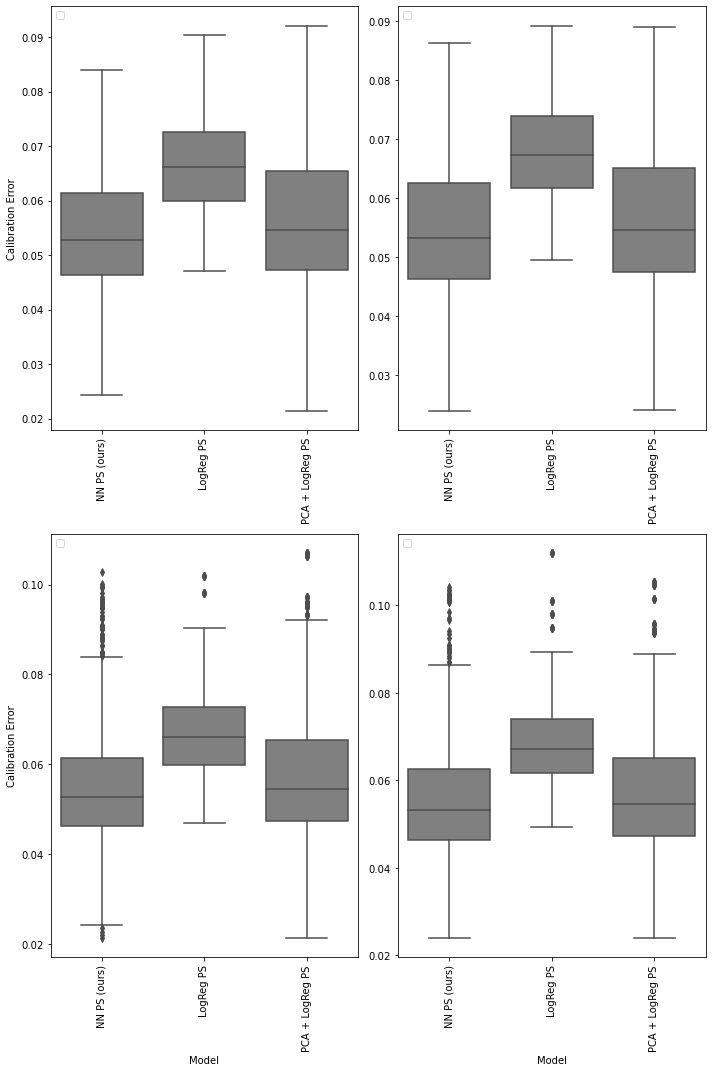}
    \caption{Calibration error boxplots on the ACIC2016 dataset: in-sample (left) and hold-out (right), without (up) and with (bottom) outliers.
    The data points underlying this figure refer to the average calibration error across a dataset version, corresponding to a single draw of the random seed, and a training seed.
    }
    \label{fig:acic2016_calibration_errors}
\end{figure}

\begin{figure}[ht!]
    \centering
    \includegraphics[width=0.8\linewidth]{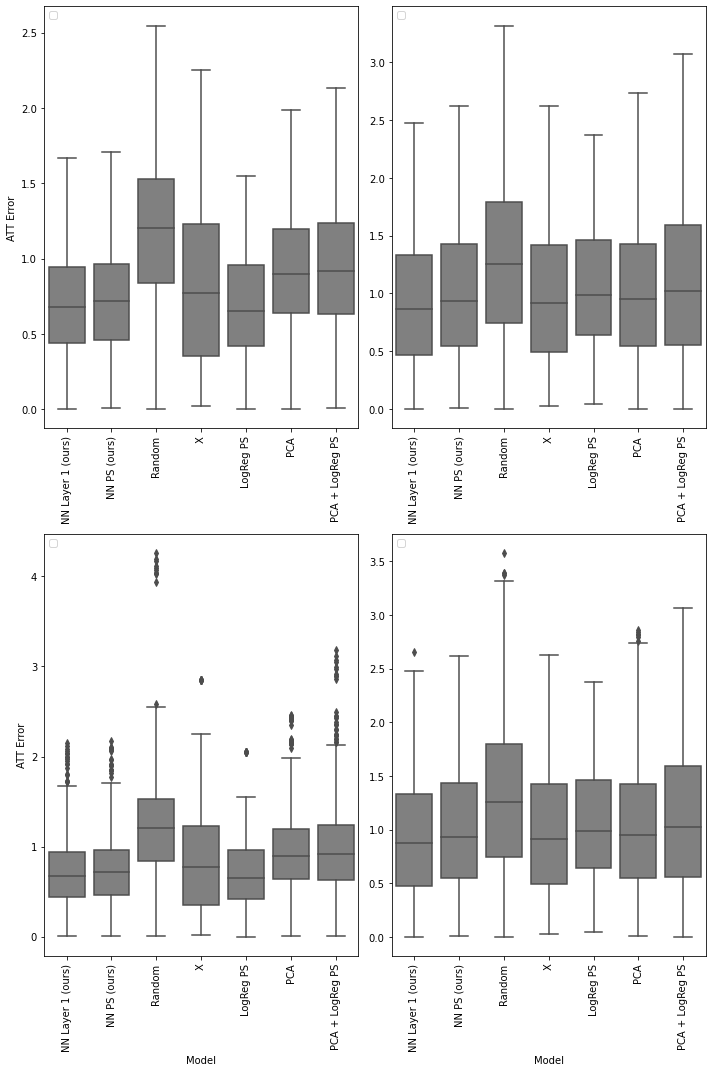}
    \caption{ATT error boxplots on the ACIC2016 dataset: in-sample (left) and hold-out (right), without (up) and with (bottom) outliers.
    The data points underlying this figure refer to the ATT computed on a dataset version, corresponding to a single draw of the random seed, and a training seed.}
    \label{fig:acic2016_att_errors}
\end{figure}

\begin{figure}[ht!]
    \centering
    \includegraphics[width=0.8\linewidth]{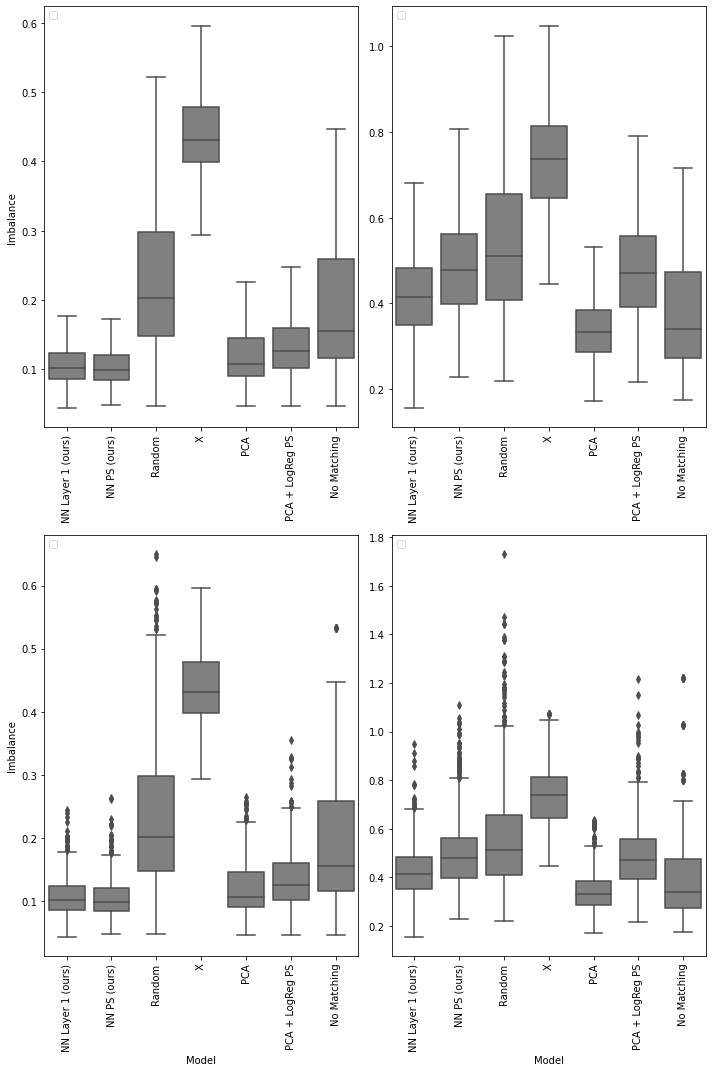}
    \caption{Sample imbalance boxplots on the ACIC2016 dataset: in-sample (left) and hold-out (right), without (up) and with (bottom) outliers.
    The data points underlying this figure refer to sample imbalance computed on a dataset version, corresponding to a single draw of the random seed, and a training seed.}
    \label{fig:acic2016_imbalances}
\end{figure}

\begin{figure}[ht!]
    \centering
    \includegraphics[width=0.8\linewidth]{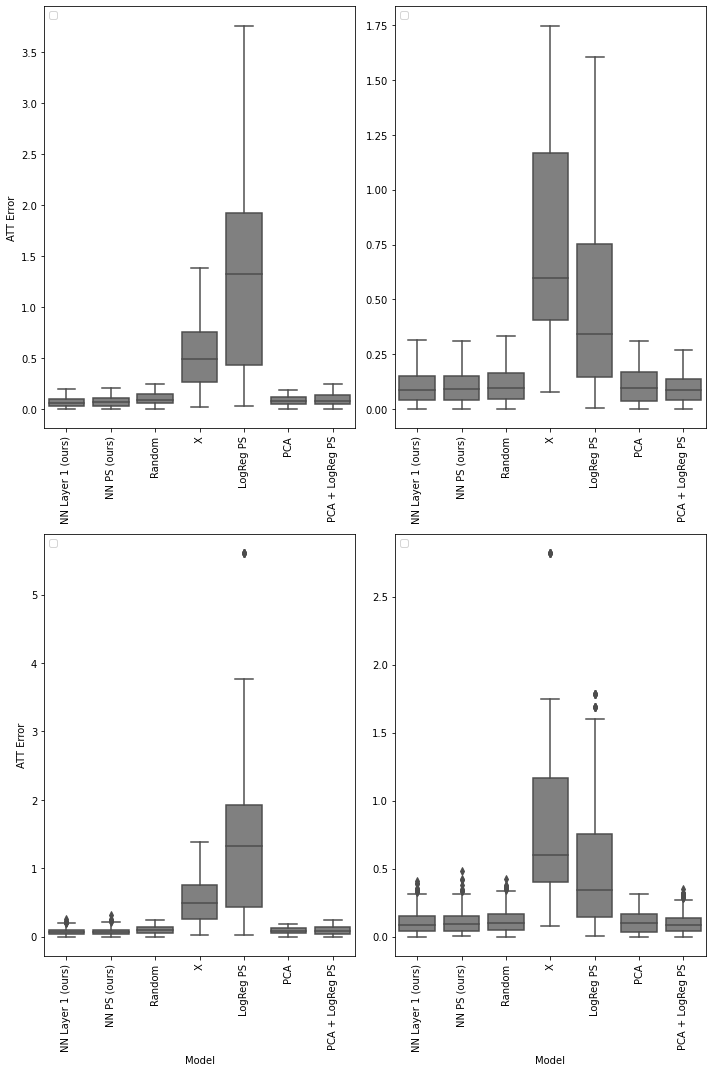}
    \caption{ATT error boxplots on the News dataset: in-sample (left) and hold-out (right), without (up) and with (bottom) outliers.
    The data points underlying this figure refer to the ATT computed on a dataset version, corresponding to a single draw of the random seed, and a training seed.}
    \label{fig:news_att_errors}
\end{figure}

\begin{figure}[ht!]
    \centering
    \includegraphics[width=0.8\linewidth]{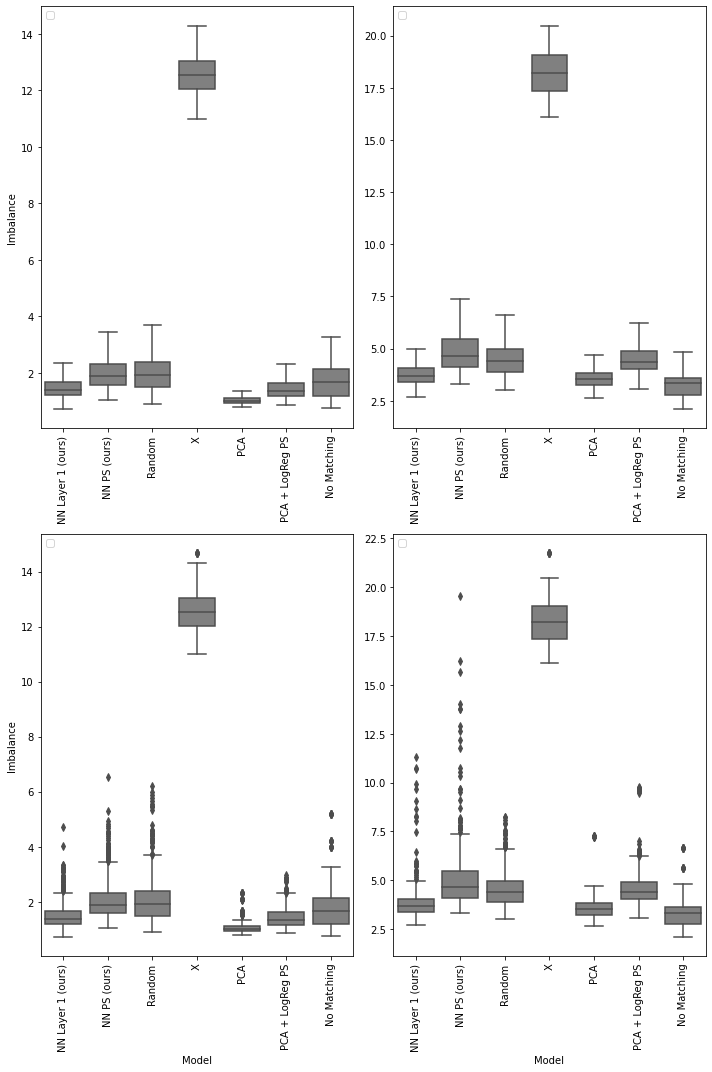}
    \caption{Sample imbalance boxplots on the News dataset: in-sample (left) and hold-out (right), without (up) and with (bottom) outliers.
    The data points underlying this figure refer to sample imbalance computed on a dataset version, corresponding to a single draw of the random seed,  and a training seed. Note that we do not show the boxplot for LogReg PS, whose exceptionally high values were hindering the readability of the Figure.}
    \label{fig:news_imbalances}
\end{figure}

\begin{figure}[ht!]
    \centering
    \includegraphics[width=0.8\linewidth]{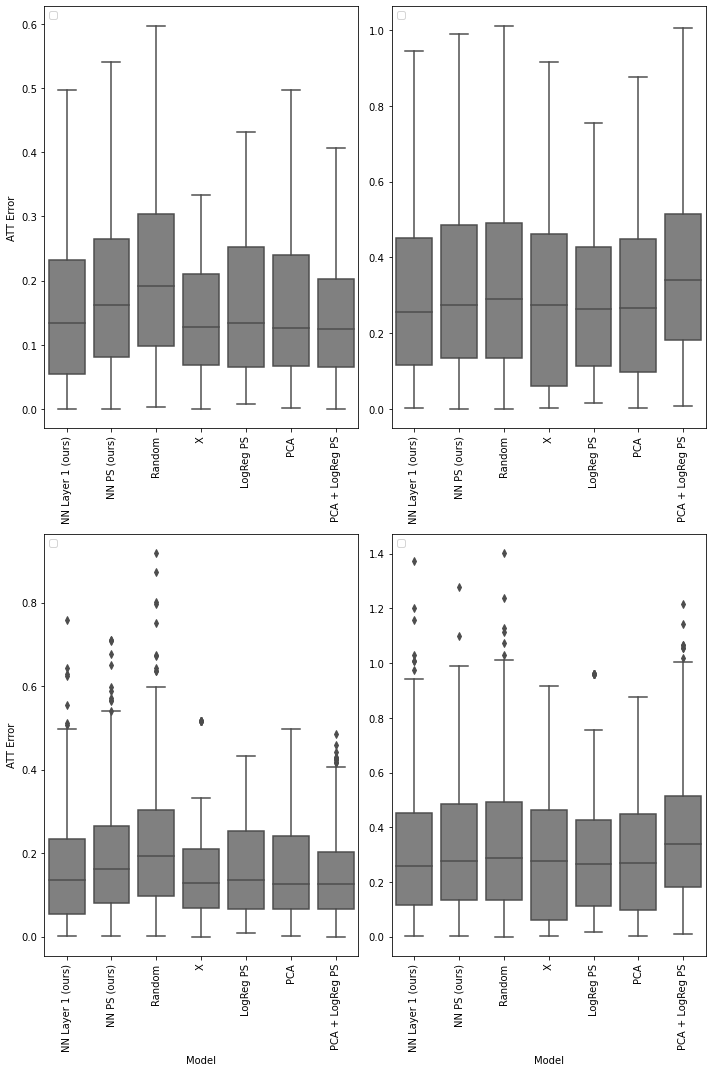}
    \caption{ATT error boxplots on the IHDP dataset: in-sample (left) and hold-out (right), without (up) and with (bottom) outliers.
    The data points underlying this figure refer to the ATT computed on a dataset version, corresponding to a single draw of the random seed,  and a training seed.}
    \label{fig:ihdp_att_errors}
\end{figure}

\begin{figure}[ht!]
    \centering
    \includegraphics[width=0.8\linewidth]{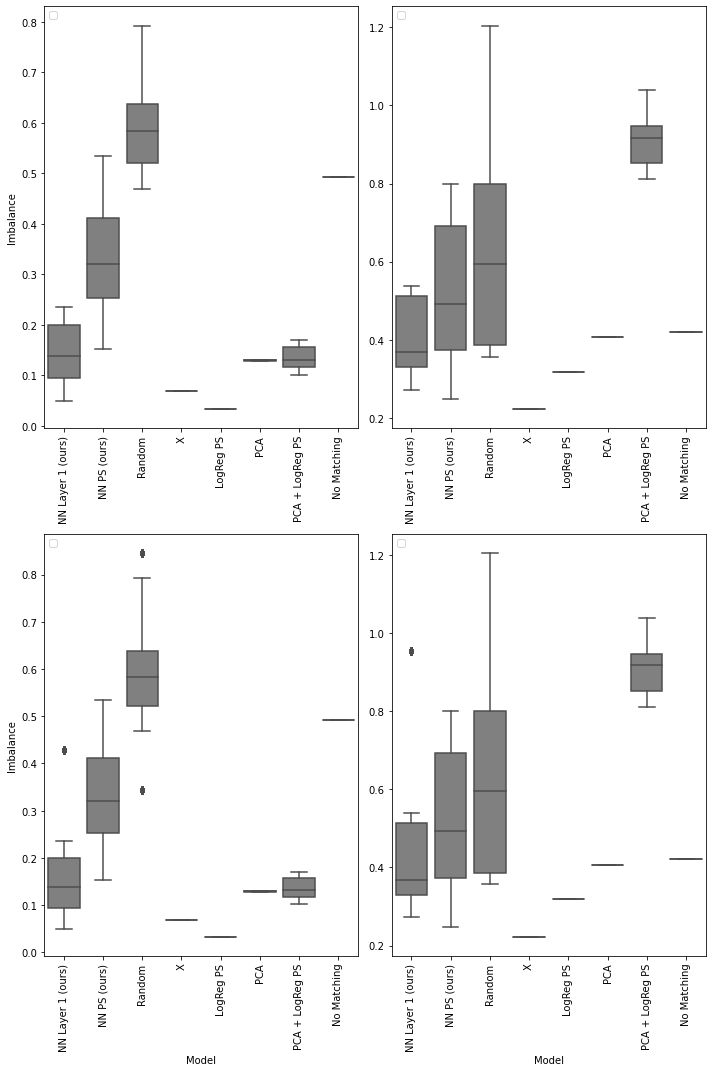}
    \caption{Sample imbalance boxplots on the IHDP dataset: in-sample (left) and hold-out (right), without (up) and with (bottom) outliers.
    The data points underlying this figure refer to sample imbalance computed on a dataset version, corresponding to a single draw of the random seed,  and a training seed.}
    \label{fig:ihdp_imbalances}
\end{figure}

\vfill

\end{document}